\documentclass[final,main]{article}

\PassOptionsToPackage{numbers,sort&compress}{natbib}

\usepackage{neurips_2026}
\usepackage[utf8]{inputenc} 
\usepackage[T1]{fontenc}    
\usepackage{multicol}
\usepackage{multirow}
\usepackage{microtype}
\usepackage{enumitem}
\usepackage{graphicx}
\usepackage{caption}
\usepackage{subcaption}
\usepackage[linesnumbered,ruled,vlined,commentsnumbered]{algorithm2e}
\usepackage{algpseudocode}
\usepackage{subcaption}
\usepackage{booktabs} 
\usepackage{booktabs}
\usepackage{colortbl}
\usepackage{diagbox}
\usepackage[table]{xcolor}
\usepackage{pgf}
\definecolor{asrLow}{RGB}{255,245,245}
\definecolor{asrMid}{RGB}{255,220,220}
\definecolor{asrHigh}{RGB}{255,190,190}

\definecolor{asrBlueLow}{RGB}{245,248,255}
\definecolor{asrBlueMid}{RGB}{220,230,255}
\definecolor{asrBlueHigh}{RGB}{190,210,255}

\definecolor{acclow}{RGB}{203, 249, 215}

\usepackage[colorlinks,citecolor=red,urlcolor=blue,bookmarks=false,hypertexnames=true]{hyperref} 
\usepackage[toc,page]{appendix}
\usepackage{amsmath}
\usepackage{amssymb}
\usepackage{mathtools}
\usepackage{amsthm}
\usepackage{nicefrac}       

\usepackage[capitalize,noabbrev]{cleveref}

\newtheorem{theorem}{Theorem}
\newtheorem*{theorem*}{Theorem}
\newtheorem{lemma}[theorem]{Lemma}
\newtheorem{proposition}[theorem]{Proposition}
\newtheorem*{proposition*}{Proposition}
\newtheorem{definition}[theorem]{Definition}
\newtheorem{remark}[theorem]{Remark}
\newtheorem{corollary}[theorem]{Corollary}
\newtheorem{assume}[theorem]{Assumption}

\usepackage[toc,page]{appendix}
\usepackage[most]{tcolorbox}
\tcbuselibrary{theorems}

\tcbset{
  mytheobox/.style={
    enhanced,
    breakable,
    colback=gray!10!white,
    leftrule=1mm,
    toprule=0pt,
    bottomrule=0pt,
    rightrule=0pt,
    arc=0pt,
    left=0.7mm, right=0.7mm,
    top=0.7mm, bottom=0.7mm,
    before skip=7pt,
    after skip=7pt, 
  }
}

\newtcolorbox{theoremshadedbox}[1][]{mytheobox, title=#1}

\newtcolorbox{shadedtheorem*}[1][]{
  enhanced,
  breakable,
  colback=gray!5!white,
  colframe=gray!50!black,
  sharp corners,
  boxrule=0.6pt,
  fonttitle=\bfseries,
  title=Theorem,
  attach boxed title to top left={yshift=-1mm,xshift=2mm},
  boxed title style={colback=gray!20!white},
  #1
}

\newtcolorbox{propositionshadedbox}[1][]{mytheobox, title=#1}

\newtcolorbox{lemmashadedbox}[1][]{mytheobox, title=#1}

\newtcolorbox{definitionshadedbox}[1][]{mytheobox, title=#1}

\newtcolorbox{remarkshadedbox}[1][]{mytheobox, title=#1}

\newtcolorbox{corollaryshadedbox}[1][]{mytheobox, title=#1}

\newtcolorbox{assumptionshadedbox}[1][]{mytheobox, title=#1}

\newtcolorbox{shadedequation}[1][]{colback=gray!5!white, colframe=gray!50!black, boxrule=0.6pt, sharp corners, title=#1, coltitle=black, enhanced}

\newcommand{\cn}{\mathbb{C}^n}
\newcommand{\cnn}{\mathbb{C}^{n\times n}}
\newcommand{\cNn}{\mathbb{C}^{N\times n}}
\newcommand{\rn}{\mathbb{R}^n}

\newcommand{\Rnn}{\mathbb{R}^{n\times n}}

\newcommand{\fn}{\mathbb{F}^n}
\newcommand{\fN}{\mathbb{F}^N}

\newcommand{\fNn}{\mathbb{F}^{N\times n}}

\newcommand{\xadv}{x_\mathrm{adv}}
\newcommand{\edgf}{\varepsilon_{\mathrm{DGF}}}

\newcommand{\sgn}{\mathrm{sign}}

\newcommand{\loss}{\mathcal{L}}

\newcommand{\pertset}{\mathcal{C}}

\newcommand{\tilx}{x}

\usepackage{nicefrac}       
\usepackage{microtype}      

\title{Frame the adversary:\\ a structure-aware attack methodology}

\author{
Vicky Kouni\\
Paris Dauphine - PSL University\\
\texttt{vasiliki.kouni@lamsade.dauphine.fr}
\And
Stelios Perrakis\\
CentraleSupélec - Paris-Saclay University\\
\texttt{stelios.perrakis@centralesupelec.fr}
\And
Francis Bach\\
Inria - Ecole Normale Supérieure -
PSL University\\
\texttt{francis.bach@inria.fr}
\And
Pascal Frossard\\
EPFL\\
\texttt{pascal.frossard@epfl.ch}
\And
Yann Chevaleyre\\
Paris Dauphine - PSL University\\
\texttt{yann.chevaleyre@lamsade.dauphine.fr}
}

\begin{document}

\maketitle

\begin{abstract}
Frequency-based adversarial attacks have recently grown popular by exploiting spectral sensitivities shared across neural architectures. Unlike spatial perturbations, frequency-based attacks expose deeper vulnerabilities, making them especially valuable for robust evaluation of safety-critical and security-sensitive applications. Yet, existing approaches are typically not derived as solutions to an optimization problem that explicitly captures transform-domain structure. In this paper, we propose a methodology for crafting principled frequency-based adversarial attacks, via a dedicated optimization framework. A cornerstone of our method hinges on the introduction of a perturbation constraint set, tied to highly structured non-orthogonal transforms, well-known for their flexible, non-predefined frequency handling. We prove that the attacks emerge as weighted $\ell_2$-projections onto this set, yielding a general and controlled attack generation mechanism. By this, we provide a clear geometric attack characterization, ensuring alignment between the optimization objective and the perturbation constraint. We assess our framework on standardized datasets, for pretrained and adversarially robust models. Results highlight that our attacks, being solutions to an optimization problem, over a structured perturbation set, are highly effective, even across different, unseen architectures. Our methodology could serve as a theoretical baseline for designing and analyzing transformed-based attacks, targeting fundamental model vulnerabilities, instead of mere architecture-specific artifacts typically studied in the robustness literature.
\end{abstract}

\section{Introduction}
Frequency-based attacks \cite{yao2024interpretable,zhang2024fourier,wang2024boosting,zheng2025boosting} have recently emerged as a powerful paradigm in adversarial learning. By leveraging the sensitivity of models to frequency components of the data, these attacks reveal vulnerabilities that often remain hidden in the spatial domain. This helps prevent catastrophic failures in safety-critical vision systems or privacy protection, since frequency-based attacks are highly effective, even across unseen target models, thus implying high transferability \cite{qian2024enhancing,wang2024boosting,zheng2025boosting}. The latter is often attributed to the ability of frequency-based attacks to exploit shared spectral biases inherent to neural architectures \cite{fridovich2022spectral,rahaman2019spectral}, rather than overfitting to model-specific spatial features, thereby targeting common vulnerabilities that persist across diverse model designs.

Most existing approaches rely on projected gradient descent (PGD) \cite{madry2017towards}, by directly ``plugging in'' to the attack pipeline a frequency transform, e.g., Fourier \cite{zhang2024fourier}, discrete cosine \cite{jia2022exploring}, or wavelet \cite{luo2022frequency,yi2024time}, along with a selection mechanism on the transformed data. The motivation behind this mechanism hinges upon selectively amplifying or suppressing low-/high-frequency components, that carry the most structural information, thus exposing vulnerabilities that remain hidden in the pixel space. In this way, manipulating spectral information offers a more efficient attack regime, enabling perturbations that achieve higher attack success rates, with comparable computational effort.

Despite their success, state-of-the-art (SotA) frequency-based attacks are not accompanied by a theoretical justification of their design, thereby carrying a strong frequency bias. What is more, such attacks do not generally emerge as solutions to an open optimization problem, which, in turn, would reflect the true nature of adversarial learning in the presence of transforms, and help explaining the effectiveness and transferability of the attacks. Still, a few recent works have proposed optimization problems for creating frequency-based adversarial attacks \cite{jia2022exploring,zhou2025numbod}, but these are application-dependent, thus rendering the generalization to other settings unclear. The interaction between the attack constraint, the transform-domain representation, and the resulting perturbation structure is thus not fully understood yet. Without a principled methodology explaining the effectiveness and transferability of frequency-based attacks, it remains elusive whether attacks should rely on certain frequency components and in what way \cite{wang2020high,zhang2024fourier,zhu2023ligaa,fu2024transferable}.

\begin{figure}[ht!]
    \centering
    \begin{subfigure}[h]{0.4\textwidth}
        \centering
        \includegraphics[width=0.95\textwidth]{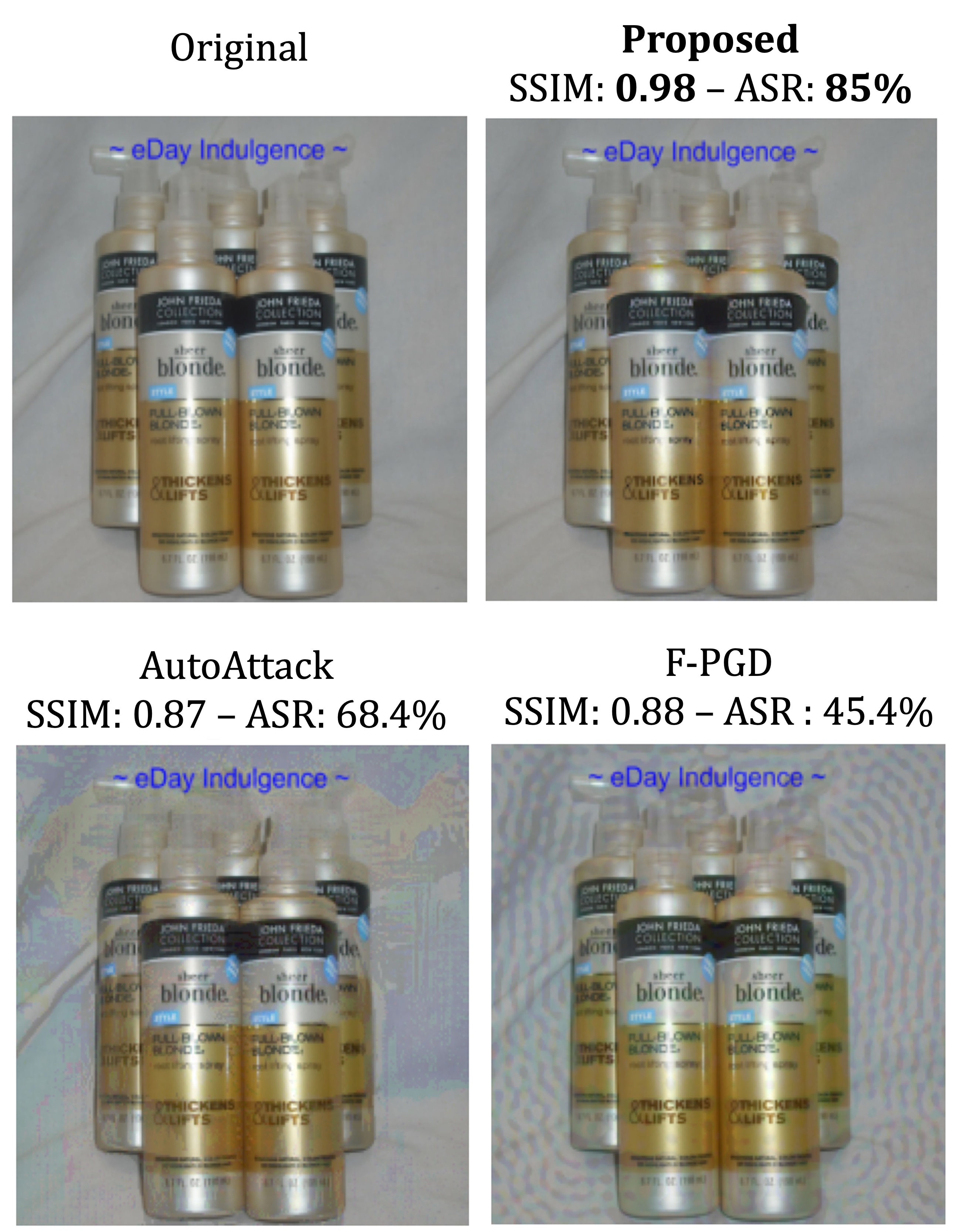}
        \caption{Example \textbf{ImageNet} adversarial images, with \textbf{SSIM} quality and \textbf{ASR} attack efficiency values. Bold letters indicate the method with the best performance in terms of both evaluation metrics.}
        \label{shampoos}
    \end{subfigure}%
    \hfill
    \begin{subfigure}[h]{0.55\textwidth}
    \centering
    \includegraphics[width=0.95\textwidth]{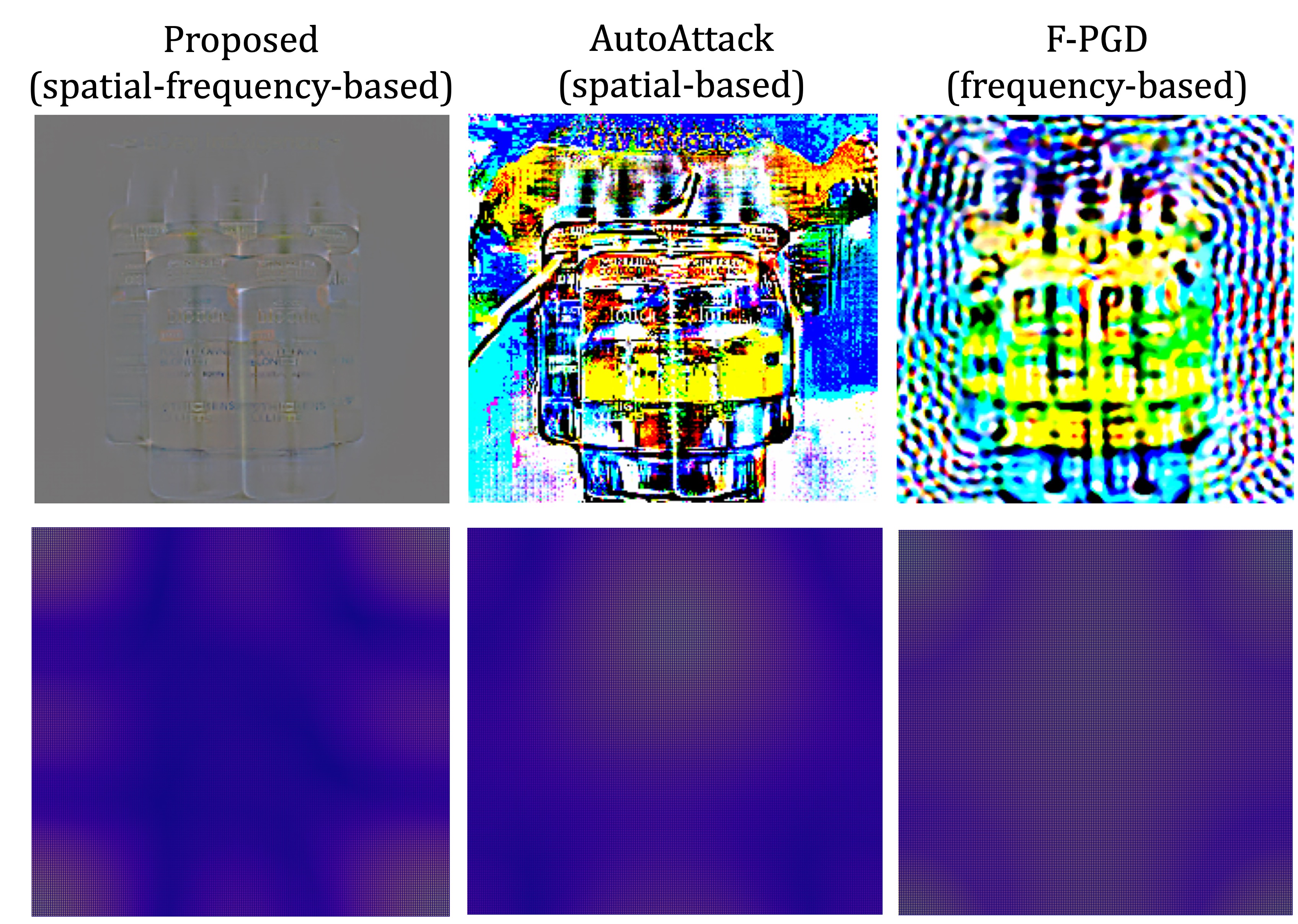}
    \caption{\textbf{Spatial (top)} and \textbf{frequency (bottom)} representation of the \textbf{proposed attack} and the two \textbf{baselines}, corresponding to the adversarial examples of Fig.~\ref{shampoos}. In contrast to the baselines, our attack enjoys strategically-placed energy localization, combining the ``best of both worlds''. This allows it to align with the model's spectral sensitivities, exploiting critical feature extraction pathways to achieve the best attack success rate.}
    \label{spectrashampoos}
    \end{subfigure} 
    \caption{\footnotesize Evaluation of the baselines ($\varepsilon=8/255$) and our attack ($\edgf=\eta8/255$), on an exemplary RobustBench model \cite{amini2024meansparse}.}
    \label{imagenet}
\end{figure}

Motivated by recent advances in frequency-based adversarial attacks, and by the lack of a principled method and solid understanding of their underpinnings, in this paper, we aim to address the following core question:
\begin{equation}
\label{q1}
\mbox{\textit{What are the underlying mechanisms of success of frequency-based adversarial attacks?}} \tag{$\dagger$}
\end{equation}
To that end, our main contributions are as follows:
\begin{itemize}
    \item We propose \textit{a new attack methodology, based on overcomplete (non-orthogonal) spatial-frequency (SF) transforms}, which are popular tools for image-processing tasks \cite{paukner2007foundations,liu2024gabor}. The overcompleteness enables flexible and structured representations that orthogonal transforms like Fourier or the discrete cosine cannot capture in baseline frequency-based attacks. 
    \item Our framework consists of designing an optimization problem, with \textit{a cornerstone being the introduction of a new perturbation set} for constraining the attacks, which satisfy a weighted SF norm constraint. The latter spreads attack energy across structured SF directions, aligned with models' spectral sensitivities, making the attack more effective and transferable. We employ PGD w.r.t. the proposed set and prove that its projection yields a closed-form attack, as the solution to our optimization problem. To the best of our knowledge, this is \textit{the first principled methodology giving a convincing partial answer to \eqref{q1}}, by exploiting the geometry of the set constraining the attacks, through the structure of the SF transform.
    \item We evaluate the success of our proposed framework through a series of experiments. \textit{Our findings are three-fold}: a) our attack is highly effective on a range of different source-target models, exposing vulnerabilities even of defended models, thereby indicating high transferability; b) comparisons with SotA representative spatial- and frequency-based attacks demonstrate that our method outperforms the baselines in terms of effectiveness and transferability, while preserving the underlying image structure. This behavior showcases that attacks resulting from a well-defined optimization problem, while respecting the structure of the transforms involved, can enjoy improved performance, without sacrificing visual quality; c) the observed transferability sheds light on the underpinnings of frequency-based attacks: structuring the perturbation set via geometry-aligned transforms, is key to capturing spectral sensitivities shared across architectures, bolstering our theoretical framework.
\end{itemize}

Overall, our principled methodology illuminates -- beyond common high/low frequency assumptions -- the effectiveness and transferability of frequency-based adversarial attacks, offering a flexible alternative to manually designed frequency selection strategies in the attacks. This enables a systematic study of how properties of involved transforms influence adversarial vulnerability. Our framework could be deployed as a testbed for developing more robust evaluation protocols and defenses that explicitly account for transform-domain structure, e.g., benchmarking robustness for structured, transformed-based perturbation sets. Finally, it could contribute to the development of a more holistic understanding of models' vulnerabilities to data perturbations, towards building safer AI systems.

\section{Related work}
\label{freattacks}
\textbf{Spatial-based attacks.} 
For a loss function $\mathcal{L}(\cdot,\cdot)$ and input-output data pairs $(x,y)$, a standard regime for creating adversarial attacks in the pixel space consists of solving
\begin{equation}
    \begin{split}
    \label{maxattack}
    \max_{\delta\in\rn}\,&\mathcal{L}(x+\delta,y)\quad\text{such that }\quad\|\delta\|_p\leq\varepsilon.
    \end{split}
\end{equation}
The popular method of PGD \cite{madry2017towards} relies on the projection on the $\ell_p$-ball to yield the solution of \eqref{maxattack} after $k=0,\dots,K-1$ iterations, i.e.,
\begin{equation}
\label{deltak}
    \delta^{k+1}=\Pi_{\|\cdot\|_p}(\delta^k+\gamma\cdot\mathrm{norm}(\nabla_\delta\loss(x+\delta^k,y))).
\end{equation}
Here, $\Pi_{\|\cdot\|_p}(\cdot)$ is the projection operator on the $\ell_p$-ball, $\gamma>0$ is the step size, and $\mathrm{norm(\cdot)}$ pertains to some form of normalization for the gradients, to ensure that each update moves optimally within the allowed $\ell_p$-ball. For instance, if $p=2$, then $\mathrm{norm}(\nabla_\delta\loss(x+\delta^k,y))=\nabla_\delta\loss(x+\delta^k,y)/\|\nabla_\delta\loss(x+\delta^k,y)\|_2$ while for $p=\infty$, it holds $\mathrm{norm}(\nabla_\delta\loss(x+\delta^k,y))=\sgn(\nabla_\delta\loss(x+\delta^k,y))$. This implies that the gradients are normalized by using the dual norm of $\|\cdot\|_p$. Finally, the adversarial image is $x_\mathrm{adv}=x+\delta^K$, clipped in $[0,1]$ to fall within a valid image range.

\textbf{Frequency-based attacks.} 
The PGD framework can be extended to attacks relying on a frequency representation of $x$. SotA frequency-based attacks take into account the sensitivity of models to high/low frequencies, and create PGD-type attacks, which rely on some orthogonal frequency transform $D\in\Rnn$, e.g., wavelet \cite{freqattack,fu2024transferable}, discrete cosine \cite{yao2024interpretable,wang2024boosting,zheng2025boosting}, or Fourier \cite{li2021f,zhang2024fourier,qian2024enhancing}. Then, a frequency-based counterpart of \eqref{deltak} is generally formulated as
\begin{equation}\label{freqattack}
    \delta^{k+1}=\Pi_{\|\cdot\|_p}(\delta^k+\gamma\cdot\mathrm{norm}(g^k)),
\end{equation}
where
\begin{equation}
\label{maskgrad}
    g^k=\nabla_\delta\mathcal{L}(D^{-1}\mathrm{mask}(D(x+\delta^k)),y),
\end{equation}

with the resulting adversarial image being $\xadv=x+\delta^K$, clipped in $[0,1]$, similarly to the spatial attacks. In \eqref{maskgrad}, $\mathrm{mask}(\cdot)$ denotes a mechanism selecting a particular frequency range to which models are sensitive, to increase the effectiveness of the attack, or make it easily transferable across target models. This design is motivated by observations that neural networks exhibit spectral biases \cite{fridovich2022spectral,rahaman2019spectral}, making them more vulnerable to perturbations in certain frequency ranges. For instance, $\mathrm{mask}(\cdot)$ can be a random multiplicative or binary mask \cite{yao2024interpretable,zhang2024fourier}, or some nonlinear transform selecting the high-/low-frequency components of the input image \cite{wang2024boosting,duan2021advdrop}. 

Despite their empirical success, frequency-based attacks of the form \eqref{maskgrad} are not accompanied by a theoretical justification of their design, since the data's representation w.r.t. $D$ is directly plugged in to the gradient of the loss. Consequently, the resulting perturbation is not guaranteed to solve a maximization problem that explicitly incorporates transform-domain structure. Moreover, standard $\ell_p$-constraints do not reflect the geometry induced by the transform, leading to a mismatch between the constraint set and the attack construction.

Recent works have taken steps towards addressing these limitations by formulating optimization problems that include structure in the frequency domain, via task-specific masking strategies. While effective in specific applications like face forgery \cite{jia2022exploring} or object detection \cite{zhou2025numbod}, these methods lack a general framework and do not clearly characterize the constraint sets onto which PGD projects. As a result, deriving frequency-based attacks as solutions to an appropriately defined optimization problem, satisfying a structured perturbation constraint, stays largely unexplored.
\section{Background}
\label{back}
\paragraph{Notation.} We write $\mathbb{F}$ for $\mathbb{R}$ or $\mathbb{C}$. We denote with $\Psi^*$ the conjugate transpose, i.e., $\Psi^*=\overline{\Psi}^T$, and with $\Psi^+=(\Psi^*\Psi)^{-1}\Psi^*$ the pseudo-inverse (Moore Penrose inverse) of $\Psi\in\mathbb{C}^{N\times n}$; the real part of $\Psi$ is $\Re(\Psi)$.

\textbf{Frames vs. orthogonal bases.} While orthogonal bases provide elegant data representations, overcomplete bases, dubbed frames \cite{casazza2012finite}, offer crucial data processing advantages. A frame is an overcomplete set that spans the ambient space with more vectors than the space's dimension (cf. Appendix~\ref{mainappen}). This overcompleteness can be leveraged to design the frame in a flexible fashion, to align with data geometry. For instance, SF frames are often constructed to match data geometry, hence improving performance in feature extraction \cite{luan2018gabor}, since overcompleteness enables the frame to flexibly capture localized and oriented patterns \cite{paukner2007foundations}. Next, we present operators associated to a frame, and move on to the representative class of Gabor frames; these will constitute the basis of our adversarial framework.
\begin{definition}[\cite{casazza2012finite}]
\label{frameop}
    Let $(\psi_i)_{i=1}^N$, $N\geq n$, be a frame in $\fn$. Then, for $x\in\fn$, the \textit{analysis operator} $\Psi\in\fNn$ associated to the frame is defined as $\Psi x:=(\langle x,\psi_i\rangle)_{i=1}^N$. The Grammian operator associated to the frame is defined in matrix form as $G=\Psi\Psi^*\in\mathbb{F}^{N\times N}$.
\end{definition}
\begin{remark}\label{expansion}
The analysis operator is injective, while the Grammian is rank deficient whenever $N>n$, and reveals intriguing properties of the frame: it reflects the linear dependencies among the frame elements and thus informs on its algebraic structure. Frames enable lossless data expansions: for any $x\in\fn$, it holds $x=\Psi^*(\Psi^*)^+ x=\Psi^+\Psi x$, where $\Psi^+=(\Psi^*\Psi)^{-1}\Psi^*$ is the pseudoinverse of $\Psi$, serving as its left-inverse, due to the injectivity of $\Psi$. Importantly, frame theory is extendable to more dimensions, e.g., for data $x\in\mathbb{R}^{n\times n}$, by using tensor algebra; we give more details on Appendix~\ref{mainappen}. 
\end{remark}
Given a row-normalized $\Psi\in\fNn$, its \textit{average coherence} \cite{bajwa2010gabor} is defined as
    \begin{align}
    \label{averagecoh}
    \nu(\Psi)&=\frac{1}{N-1}\max_i\bigg|\sum_{j=1,j\neq i}^N\langle\psi_i,\psi_j\rangle\bigg|,\qquad i=1,\dots,N.
    \end{align}
    Average coherence relies on the Grammian and reflects a frame's geometry, by measuring the spread of the rows within the $n$-dimensional unit ball: the smaller the coherence, the more spread out the row vectors. On the contrary, large values of $\nu(\Psi)$ pinpoint to geometric degeneracy, e.g., vectors are poorly distributed on the ambient space, due to the large number of linear dependencies among them.
\begin{definition}[\cite{paukner2007foundations}]
\label{dgf}
    A \textit{discrete Gabor frame} (DGF) results from SF shifts of a nonzero vector $g\in\fn$. The shifts are controlled by the Gabor SF parameters $\alpha,\beta>0$, $\alpha\beta<n$, so that the frame elements are $ \psi_{k,p,j}=\exp(2\pi i\beta pj/n)g(j-\alpha k)$, where $k=0,\dots,(n/\alpha)-1$, $p=0,\dots,(n/\beta)-1$, $j=0,\dots,n-1$. By creating a grid over $p,k$, and representing them as a single index $i$, we obtain $i=0,\dots,N-1$, where $N=n^2/\alpha\beta$, so that we can equivalently write $\psi_{k,p,j}=\psi_{i,j}$. Then, the associated analysis operator is $\Psi\in\mathbb{C}^{N\times n}$, whose row elements are $\psi_i^*\in\cn$, $i=0,\dots,N-1$.
\end{definition}
We note here that the Gabor parameters $\alpha,\beta$ crucially affect the representation ability of $\Psi$ \cite{kouni2023star}, and choosing them respectively is a challenging task \cite{handbook}. Moreover, by Remark~\ref{expansion}, we obtain $x=\Psi^+\Psi x$.

\textbf{Why Gabor frames?} 
Unlike global orthogonal transforms like Fourier/discrete cosine, DGFs provide a flexible SF data representation \cite{christensen2003introduction}, with overcompleteness enabling simultaneous localization across space, frequency, and orientation, hence rendering them as great candidates for representing edges and directional patterns that naturally arise in images \cite{luan2018gabor,filus2024evaluating}. This adaptive representation to structured and overlapping spectral patterns makes overcompleteness particularly important to the adversarial setting: by expressing attacks in the DGF domain, one can shape them to align with spectral sensitivities shared across different network architectures. What is more, the close connection between Gabor representations and visual processing mechanisms \cite{jones1987evaluation,daugman1985uncertainty} helps explain why Gabor frames induce norms closely related to human perception \cite{balle2012subjective}, which spatial-$\ell_p$ norms do not; this further motivates the construction of a DGF-based attack methodology, which we present in Section~\ref{main}.

\section{Main results}
\label{main}
\begin{algorithm}[t!]
\footnotesize
\SetAlgoLined
\SetAlgoNoEnd
\SetKwData{Left}{left}\SetKwData{This}{this}\SetKwData{Up}{up}\SetKwFunction{Union}{Union}\SetKwFunction{FindCompress}{FindCompress}\SetKwInOut{Input}{Input}\SetKwInOut{Output}{Output}
\Input{Original image $x$, DGF analysis operator $\Psi$, attack level $\varepsilon>0$}
\Output{Adversarial image $\xadv$}
Calculate $\Re(M_D)$ and $(\Re(M_D))^{-1}$; initialize $\delta_0$\\
\For{$k=0,\dots,K-1$}{
$\Tilde{g}^k=\nabla_\delta\loss(x+\delta^k,y)$\;
$\delta^{k+1}=\Pi_{\pertset_{M_D}}\bigg(\delta^k+\gamma\frac{(\Re(M_D))^{-1}\Tilde{g}^k}{\sqrt{(\Tilde{g}^k)^T(\Re(M_D))^{-1}\Tilde{g}^k}}\bigg)$\;}
$\xadv=\mathrm{clip}(x+\delta^K,0,1)$
\caption{DGF-PGD}
\label{dgfpgd}
\end{algorithm}
\begin{definition}
    \label{admissible}
    For $\varepsilon>0$ and $\Psi\in\fNn$ being a DGF analysis operator, we say that $\delta\in\fn$ is a spatial attack of level $\varepsilon$ if $\|\delta\|_{\fn}\leq\varepsilon$, w.r.t. some norm $\|\cdot\|_{\fn}$ in $\fn$. Additionally, we say that $w\in\fN$ is an admissible frame attack of level $\varepsilon$ if $\|w\|_{\fN}\leq\varepsilon$ and $w=\Psi\delta$, for some $\delta\in\fn$.
\end{definition}
\begin{lemma}
    [Proof in Appendix~\ref{admappen}]
\label{admlemma}
For $\delta\in\fn$ and $\varepsilon>0$, let $w=\Psi\delta\in\fN$, with $\Psi\in\fNn$ being a DGF analysis operator. If $\delta\in\fn$ is a spatial attack of level $\varepsilon$, then $w=\Psi\delta$ is an admissible frame attack of level $\|\Psi\|_{\mathrm{op}}\varepsilon$, where $\|\Psi\|_{\mathrm{op}}=\max_{\|\delta\|_{\fn}\leq1}\|\Psi\delta\|_{\fN}$. Additionally, if $w\in\fN$ is an admissible frame attack of level $\varepsilon$, then $\delta=\Psi^+w$ is a spatial attack of level $\|\Psi^+\|_{\mathrm{op}}\varepsilon$, where $\|\Psi^+\|_{\mathrm{op}}=\max_{\|w\|_{\fN}\leq1}\|\Psi^+w\|_{\fn}$.
\end{lemma}
The purpose of Lemma~\ref{admlemma} is to show that any admissible attack in the frame domain is induced by a suitable attack in the spatial domain and vice versa.

\textbf{Threat model. }We operate in a white-box regime, where the adversary has access to the model, input data and gradients, and the loss function. In Section~\ref{results}, we also evaluate transfer-based black-box attacks, where adversarial examples crafted on a source model are evaluated on an unseen one. This combination assesses attack effectiveness, even beyond the source model, addressing \eqref{q1}. Finally, we introduce a structured, DGF-based perturbation set, whose geometry promotes SF-localized attacks.

\textbf{Proposed adversarial attack.}  We are inspired by \citet{zhang2024fourier} and \citet{duan2021advdrop}, who propose frequency-based adversarial attacks relying on orthogonal transforms. Similarly, we derive frequency-based adversarial attacks, yet differentiate our approach in the following ways: a) we operate in the DGF (non-orthogonal) domain, where the associated analysis operator possesses inherent structure and offers greater flexibility in the data representation, b) we introduce a maximization problem over a new set of perturbations, satisfying a weighted-DGF norm constraint. We employ a variant of projected gradient ascent w.r.t. that set and prove that the associated projection operator yields a solution to our maximization problem, leading to a principled attack methodology. By Lemma~\ref{admlemma}, the attack we derive in the spatial domain induces an attack in the frame domain. We detail our complete framework below.

Let a classification loss $\mathcal{L}(\cdot,\cdot)$, and choose a DGF-based analysis operator. Then, we transform any clean image $x\in\rn$ in the DGF domain: $z=\Psi x\in\mathbb{C}^N$. For $\varepsilon>0$, we propose to solve
\begin{equation}
    \label{percmax}
        \max_{w}\,\loss(\Psi^+(z+w),y),\quad\text{such that}\quad w=\Psi\delta,\,\|w\|_D\leq\varepsilon,
\end{equation}
where $w\in\mathbb{C}^N$ is an admissible frame attack of level $\varepsilon$ (cf. Definition~\ref{admissible}) w.r.t. $\|\cdot\|_D$ in $\mathbb{C}^N$, which we define to be a weighted Gabor-based norm:
\begin{equation}
    \label{perc1}
    \|w\|_D=\sqrt{w^*Dw}.
\end{equation}
$D\in\mathbb{R}^{N\times N}$ is a diagonal weighting matrix that constitutes a function of the average coherence \eqref{averagecoh}:
\begin{equation}
    \label{percmetric}
    d_i=\left(\tau+(N-1)^{-1}{\Bigg|\sum_{j=1}^N\langle\psi_i,\psi_j\rangle\Bigg|}\right)^{-1},\qquad i=1,\dots,N,\ \tau>0.
\end{equation}
Now, we rewrite \eqref{perc1} as $\|w\|_D=\|\Psi\delta\|_D=\sqrt{\delta^TM_D\delta}$, where
\begin{equation}
\label{mmatrix}
    M_D=\Psi^*D\Psi\in\cnn.
\end{equation}
\textbf{Why is \eqref{percmetric} important?} $D$ should reflect frame geometry in a similar manner to the Grammian operator (cf. Definition~\ref{frameop}). By defining $D$ as a regularized inverse proxy of the average coherence, we perform downweighting -- in a positive-definite and stable fashion due to the sufficiently small $\tau>0$ -- to prevent redundant frame elements from dominating the attack level, thereby enforcing a perturbation norm that better reflects intrinsic signal-space complexity.
\begin{figure}[t!]
    \centering
    \begin{subfigure}[h]{0.4\textwidth}
        \centering
        \includegraphics[width=0.8\textwidth]{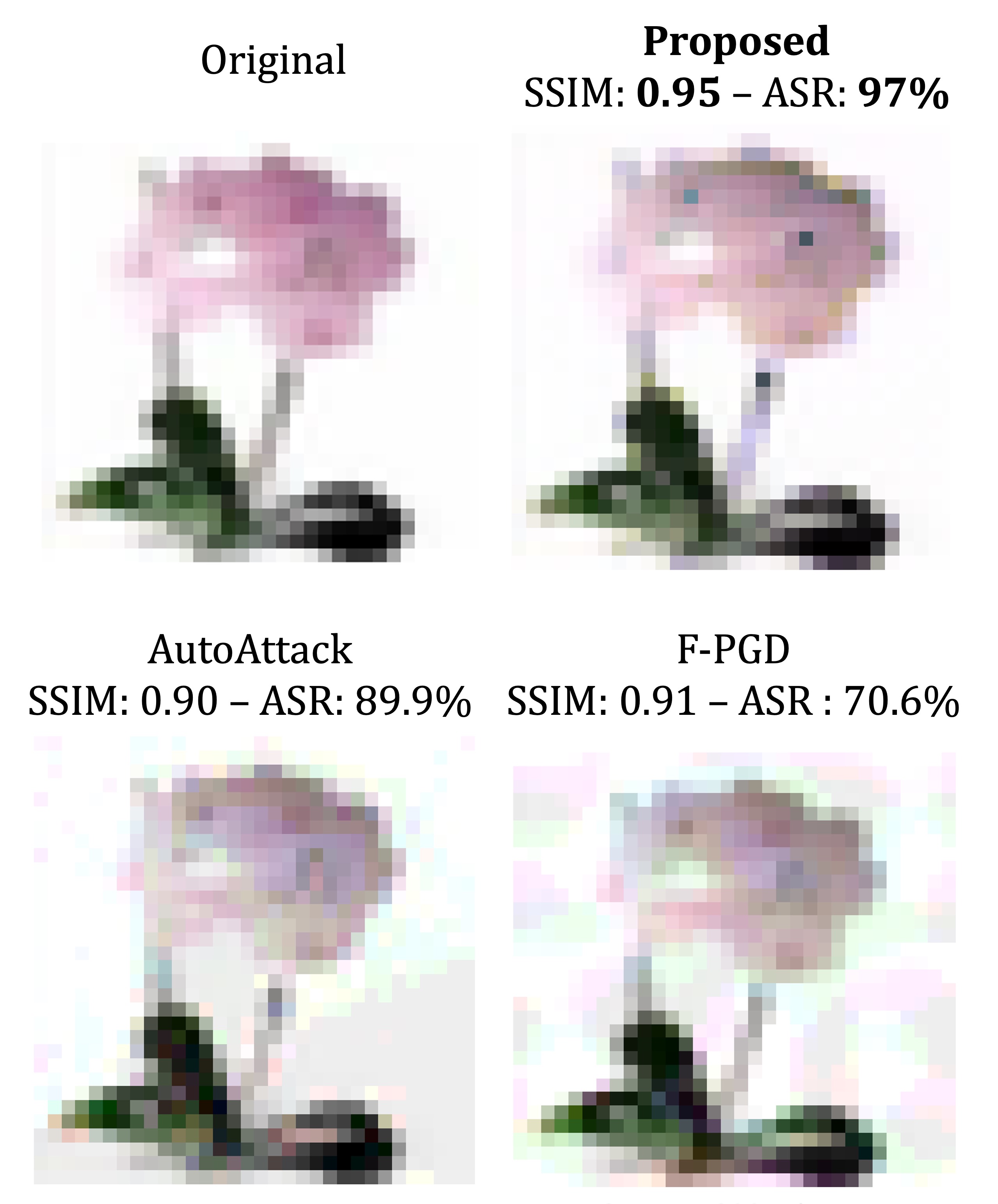}
        \caption{Example \textbf{CIFAR100} adversarial images, with \textbf{SSIM} quality and \textbf{ASR} attack efficiency values. Bold letters indicate the method with the best performance in terms of both evaluation metrics.}
        \label{flowers}
    \end{subfigure}%
    \hfill
    \begin{subfigure}[h]{0.5\textwidth}
    \centering
    \includegraphics[width=0.85\textwidth]{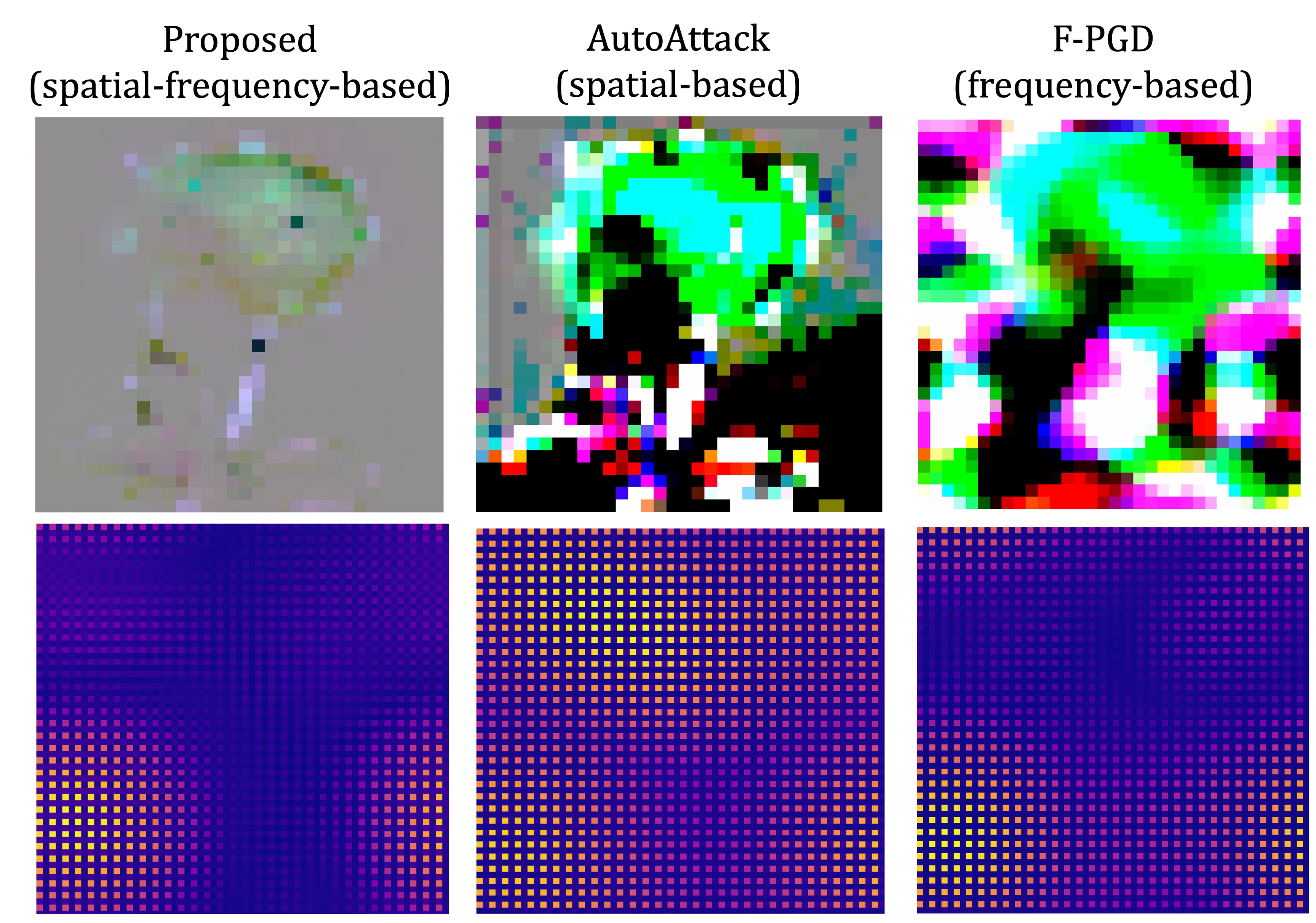}
    \caption{\textbf{Spatial (top)} and \textbf{frequency (bottom)} representation of the \textbf{proposed attack} and the two \textbf{baselines}, corresponding to the adversarial examples of Fig.~\ref{flowers}. Our attack adapts to the SF patterns arising in the images, taking advantage of spectral biases that persist across different ambient dimensions (cf. Fig.~\ref{spectrashampoos}), thereby outperforming the baselines, which lack a similar double localization.}
    \label{spectraflowers}
    \end{subfigure}
    
    \caption{\footnotesize Evaluation of the baselines ($\varepsilon=16/255$) and our attack ($\edgf=\eta\varepsilon$), on an exemplary RobustBench model \cite{chen2024data}.}
    \label{cifar100}
\end{figure}

By Remark~\ref{expansion}, $\Psi$ allows for a lossless expansion, so we get $\delta=\Psi^+w$. Since $w$ is an admissible frame attack, by Lemma~\ref{admlemma}, $\delta$ is a spatial attack, and we can easily move our optimization problem to the spatial domain; thus, \eqref{percmax} is equivalently written as 
\begin{equation}
    \begin{split}
    \label{space}
        \max_\delta&\,\loss(x+\delta,y)\quad\text{such that}\quad\delta^TM_D\delta\leq\varepsilon^2.
    \end{split}
\end{equation}
For $M_D$ as in \eqref{mmatrix}, we introduce the perturbation set
    \begin{equation}
        \label{c1}
        \pertset_{M_D}:=\{\delta\in\rn\,\mid\,\delta^TM_D\delta\leq\varepsilon^2\}.
    \end{equation}
To solve \eqref{space}, we rely on projected gradient ascent, with direction dictated by $\pertset_{M_D}$. Thus, we obtain (Proof in Appendix~\ref{pga})
\vspace{-5pt}
\begin{equation}
    \label{framepgd}
    \delta^{k+1}=\Pi_{\pertset_{M_D}}\bigg(\delta^k+\gamma\frac{(\Re(M_D))^{-1}\Tilde{g}^k}{\sqrt{(\Tilde{g}^k)^T(\Re(M_D))^{-1}\Tilde{g}^k}}\bigg),
\end{equation}
where $\Tilde{g}^k=\nabla_\delta\loss(\tilx+\delta^k)$, and $\Pi_{\pertset_{M_D}}(\cdot)$ is the projection operator on $\pertset_{M_D}$ w.r.t. the $\ell_2$-norm, given below.
\begin{proposition}
    [Proof in Appendix~\ref{projappen}]
\label{projection}
For $\Psi\in\fNn$ being the analysis operator associated to a DGF and $M_D=\Psi^*D\Psi$, the projection operator on $\pertset_{M_D}$ defined in \eqref{c1}, w.r.t. the $\ell_2$-norm, is given for any $\delta\in\rn$ as
\vspace{-5pt}
\begin{align}
\label{pidelta}
    \Pi_{\pertset_{M_D}}(\delta)=\arg\min_{\delta'\in\pertset_{M_D}}\|\delta'-\delta\|_2
    =\begin{cases}
\delta, &\delta^T M_D \delta \leq \varepsilon^2 \\
(I + 2\lambda\Re(M_D))^{-1}\delta, & \delta^T M_D \delta > \varepsilon^2
\end{cases},
\end{align}
for a uniquely determined $\lambda=\lambda(\varepsilon)>0$.
\end{proposition}
Due to \eqref{framepgd} and \eqref{pidelta}, we can also write
\begin{equation}
    \delta^{k+1}=\begin{cases}
    \widetilde{\delta^k}
 & \text{if }\widetilde{ \delta^k}\in\pertset_{M_D}\\
(I + 2\lambda\Re(M_D))^{-1}\cdot\widetilde{ \delta^k} & \text{if }\widetilde{ \delta^k}\notin\pertset_{M_D},
\end{cases}\qquad k=0,\dots,K-1,
\end{equation}
where $\widetilde{ \delta^k}=\delta^k+\gamma(\Re(M_D))^{-1}\Tilde{g}^k\cdot\Big(\sqrt{(\Tilde{g}^k)^T(\Re(M_D))^{-1}\Tilde{g}^k}\Big)^{-1/2}$. The solution $\delta_\mathrm{opt}=\delta^K$ to \eqref{space} is obtained after $K\in\mathbb{N}$ iterations of \eqref{framepgd}, 
and the adversarial image is $\xadv=\tilx+\delta_\mathrm{opt}$, appropriately clipped to fall within $[0,1]$. Overall, our proposed attack is given in Algorithm~\ref{dgfpgd}.

\textbf{Why is Proposition~\ref{projection} important?} Our main result gives a convincing partial answer to \eqref{q1} via the induced geometry of $\pertset_{M_D}$. By optimizing over $\pertset_{M_D}$, we enable attacks to align with structured SF components and exploit model sensitivities, while remaining distributed across correlated directions through frame overcompleteness, thus avoiding overfitting to pixel-space features, and improving both effectiveness on the source model and transferability to unseen models; we empirically validate this in Section~\ref{results}. The inherent overcompleteness of frames also suggests our framework could generalize to other structured transform families, which we leave as future work (cf. Section~\ref{conclusion}).
\begin{table}[t!]
\centering
\scalebox{0.7}{\begin{tabular}{lcccc}
\toprule
& \multicolumn{2}{c}{ImageNet} & \multicolumn{2}{c}{CIFAR100} \\
\cmidrule(lr){2-3} \cmidrule(lr){4-5}
Attack & Pretrained & RobustBench & Pretrained & RobustBench \\
\midrule
\textbf{Proposed} (spatial-frequency-based) & \textbf{74.5} & \textbf{63.4} & \textbf{96.5} & \textbf{85.1} \\
AutoAttack (spatial-based) & 33.6          & 30.9          & 90.1          & 44.4          \\
F-PGD (frequency-based) & 32.8          & 32.1          & 74.2          & 46.0          \\
\bottomrule
\end{tabular}}
\caption{Average \textbf{ASR} (in \%) across source-target model pairs, computed from the off-diagonal entries of the heatmaps of Fig.~\ref{transferability_imagenet_heatmaps}. \textbf{Bold letters indicate better cross-model transferability}. Results support our methodological contribution: attacks derived from a principled optimization framework, based on structured overcomplete transforms that align more naturally with data geometry, can exploit shared model vulnerabilities more effectively than purely spatial-/frequency-based methods.}
\label{mtasr}
\end{table}
\section{Experiments}
\label{exp}
We complement our principled methodology for transform-based attacks with an empirical evaluation, designed to assess its effectiveness and transferability in practice. Section~\ref{settings} features representative experimental settings across diverse pretrained and robust models, and standardized datasets; in Section~\ref{results}, we discuss corresponding results (and refer to Appendix~\ref{expappen} for more experimental details). While not intended as an exhaustive benchmark, these experiments support our theoretical motivation and lay the groundwork for broader empirical investigation (cf. Section~\ref{conclusion}).

\begin{figure}[t!]
    \centering
\begin{subfigure}[h]{0.9\textwidth}
        \centering
        \includegraphics[width=0.85\textwidth]{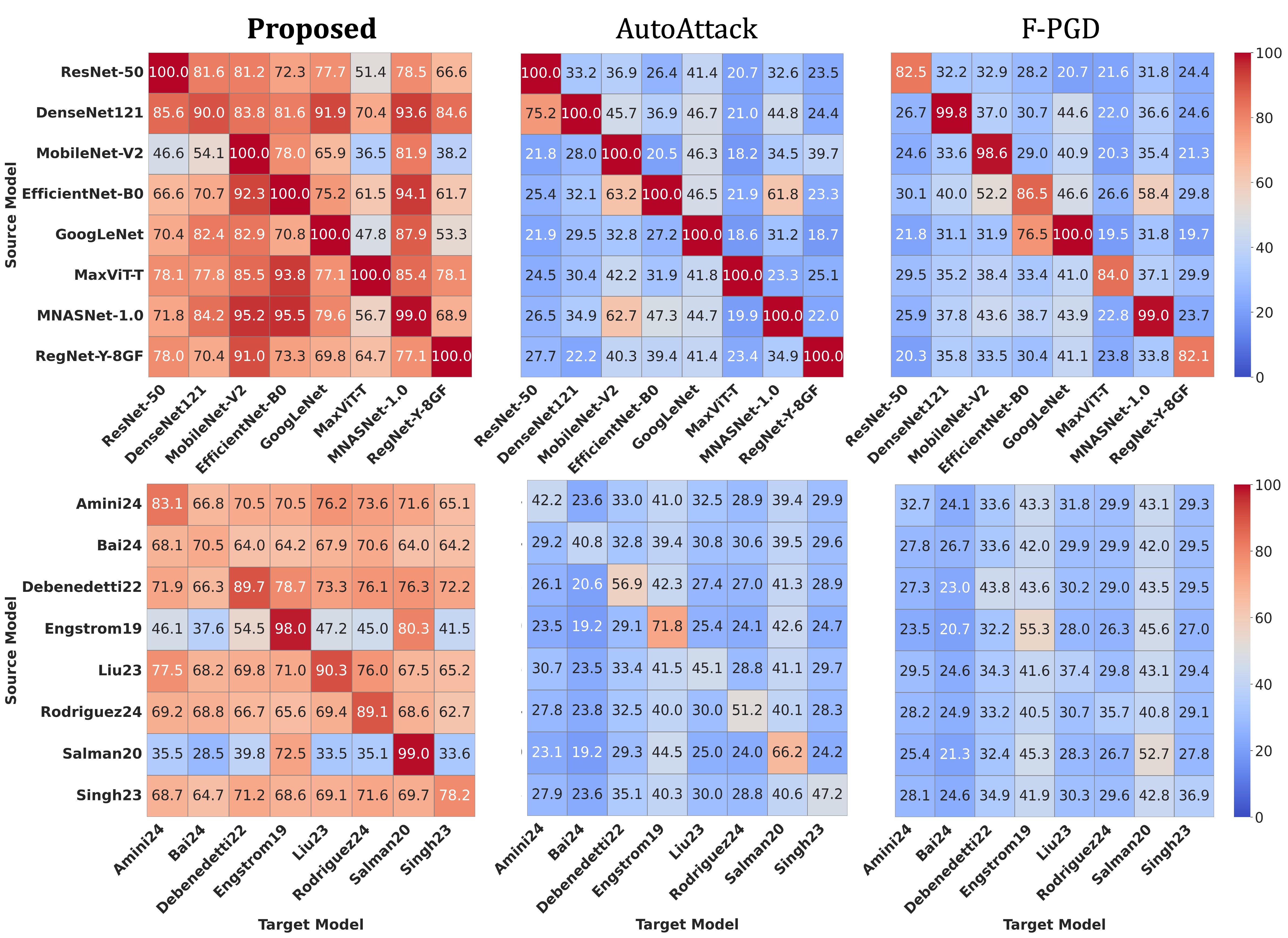}
        \caption{\textbf{ImageNet}, with $\varepsilon=4/255$ for the baselines and $\edgf=\eta4/255$ for our attack, on \textbf{pretrained (top)} and \textbf{RobustBench (bottom)} models. \textbf{Left to right: proposed attack, AutoAttack, F-PGD}.}
        \label{imagenet_trans}
    \end{subfigure}%
    
    \begin{subfigure}[h]{0.9\textwidth}
    \centering
    \includegraphics[width=0.85\textwidth]{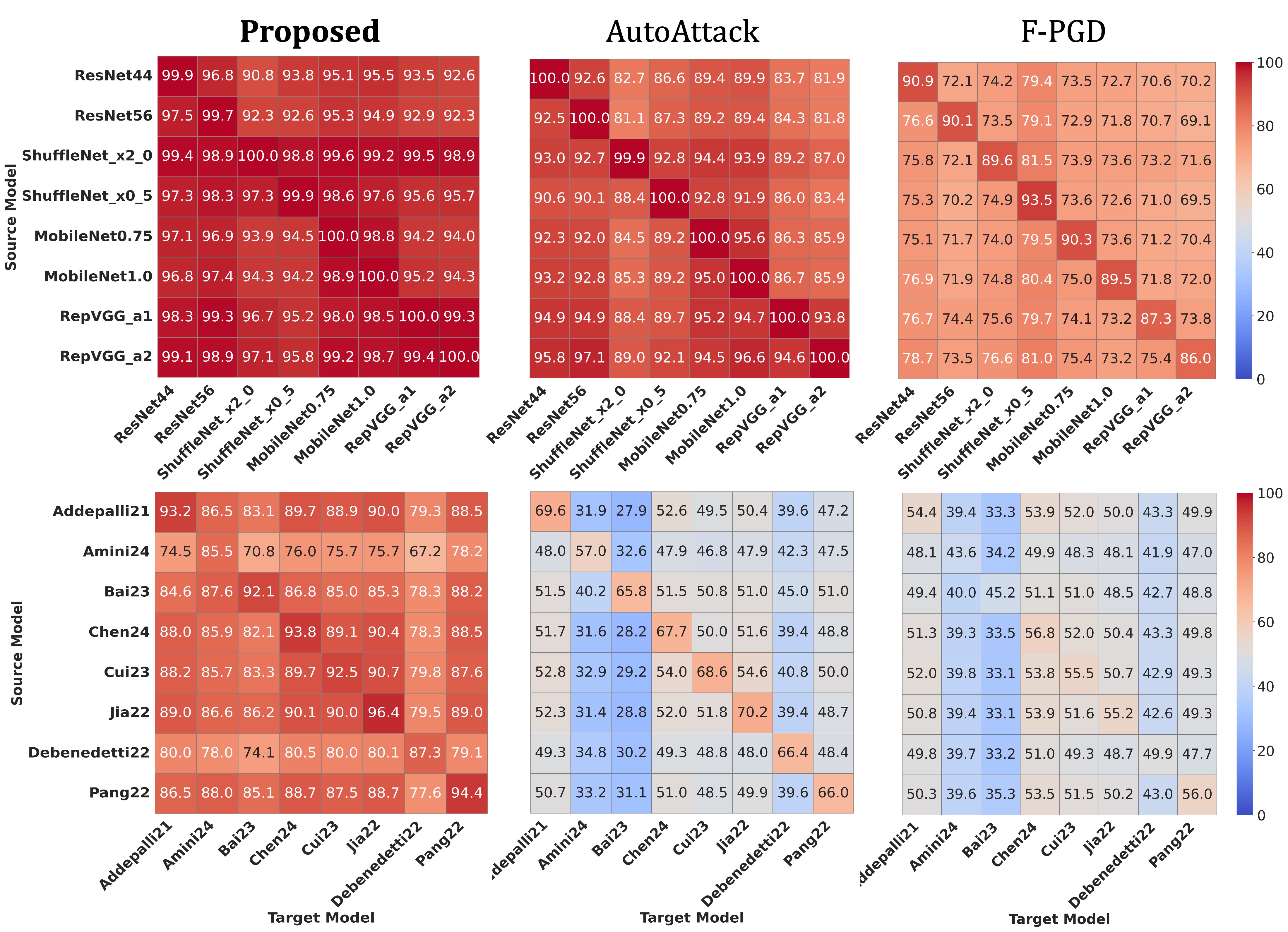}
    \caption{\textbf{CIFAR100}, with $\varepsilon=8/255$ for the baselines and $\edgf=\eta8/255$ for our attack, on \textbf{pretrained (top)} and \textbf{RobustBench (bottom)} models. \textbf{Left to right: proposed attack, AutoAttack, F-PGD}.}
    \label{cifar_trans}
\end{subfigure}
    \caption{\footnotesize \textbf{ASR} (in $\%$) for both datasets, for all attacks, evaluated on different models. \textbf{Color intensity indicates higher transferability}. Each row represents a different source model, based on which an attack is created. We observe that our proposed attack enjoys a high transferability across a variety of models, and outperforms the baselines, thereby highlighting the importance of deriving attacks from a principled methodology.}   \label{transferability_imagenet_heatmaps}
\end{figure}

\subsection{Settings}
\label{settings}
\textbf{Datasets, models \& baselines.} We evaluate our attack on $224\times224$ and $32\times32$ coloured images from ImageNet \cite{deng2009imagenet} and CIFAR100 \cite{krizhevsky2009learning}, respectively, on SotA standard pretrained \cite{chenyaofo_pytorch_cifar_models,pytorch}, and on adversarially robust models taken from the standardized benchmark of RobustBench \cite{croce2020robustbench}. Both datasets are typical benchmarks for studying robustness across different data resolutions and complexities, and model families, making them suitable candidates for assessing the behavior of our proposed methodology. We select RobustBench models having undergone both adversarial training and more elaborate defenses, i.e., sparsity-promoting \cite{amini2024meansparse} and stochastic smoothing strategies \cite{bai2024mixednuts}, and regularized adversarial training \cite{bai2024improving}, as a means of examining the performance of our attack in more advanced settings. We compare our attack to SotA baselines, being indicative of attack classes: the spatial-based AutoAttack \cite{croce2020reliable}, and the frequency-based Fourier PGD (F-PGD) \cite{zhang2024fourier}, both constrained in the $\ell_\infty$-norm ball. Since our attack results as a solution to a projection on $\pertset_{M_D}$, we aim to see how the structure induced by this closed-form solution relates to less structured attacks. For both baselines, we set the best hyperparameters proposed by the associated papers.

\textbf{Evaluation metrics \& attack parameters.} We examine varying attack levels $\varepsilon$, and measure attack effectiveness via the common attack success rate (ASR), i.e., the rate of adversarial examples that successfully fooled the model; the higher the ASR, the more vulnerable a model is. We examine visual quality via the standard structural similarity index measure (SSIM) \cite{wang2004image}, ranging in $[0,1]$; the higher the SSIM, the better the visual quality is. The perturbation set $\pertset_{M_D}$ of \eqref{c1} is an ellipsoid, not a ball. To enable a fair baseline comparison, a common practice hinges on comparing volumes of perturbation sets \cite{araujo2020advocating}, i.e., we require the ball volume (controlled by $\varepsilon)$ to be equal to the ellipsoid volume (controlled by an attack level we denote as $\edgf$) and solve for $\edgf$. Then, it can be proven that $\edgf=\eta\varepsilon$, where $\eta=\sqrt{2n/(\pi e)}(\pi n\mathrm{det}(M_D))^{1/2n}$.
\subsection{Results \& discussion}
\label{results}
\vspace{-2pt}
\textbf{The effect of $\pertset_{M_D}$ (and Gabor frames) on effectiveness and transferability.} We evaluate the effectiveness of the proposed attack on pretrained and RobustBench models of ImageNet and CIFAR100, with fixed $\edgf=\eta4/255$ and $\edgf=\eta8/255$, respectively, and report the results in the diagonal of the left heatmaps of Fig.~\ref{imagenet_trans} and Fig.~\ref{cifar_trans}, respectively, with varying architectures, e.g., from CNN to Transformers. As illustrated in the first column of Fig.~\ref{transferability_imagenet_heatmaps}, the proposed Gabor-based attack achieves nearly ideal ASR on the pretrained models, and consistently high ASR on the RobustBench models. This strong effectiveness can be owed to the structure of $\pertset_{M_D}$ \eqref{c1}, which departs fundamentally from $\ell_p$ balls. By encoding the geometry of the underlying DGF, $\pertset_{M_D}$ enables the attack to distribute energy across SF directions, being highly relevant to the models' internal representations. Overall results confirm that designing an optimization problem, endowed with a transform-aware perturbation constraint set, is a key factor for a frequency-based attack to be highly effective on different architectures, extending naturally from small- to large-scale datasets. 

The transferability of our attack is also reflected in the off-diagonal entries of the left heatmaps of Fig.~\ref{imagenet_trans} (ImageNet) and \ref{cifar_trans} (CIFAR100). For the sake of completeness, in Table~\ref{mtasr}, we also summarize the average off-diagonal ASR of Fig.~\ref{transferability_imagenet_heatmaps} -- a common practice when examining cross-model effectiveness \cite{yang2023improving,wei2022towards,wang2024boosting}. An adversarial attack is transferable when generated on a source model, yet successfully deceives unseen target models. We observe that our attack enjoys a high transferability rate, across most of the source–target model combinations, with a mild reduction for the RobustBench models. For both datasets, color intensity variations indicate that our attack generalizes beyond the adversarial defense used for each source model, even under more advanced schemes than standard adversarial training. We attribute this effect to the representation power of DGFs (cf. Section~\ref{back}), to handle SF patterns that naturally arise in images, so the attacks greatly exploit shared spectral biases of different models; this enhances transferability. Therefore, with a DGF-based attack, we hit models at their core: their sensitivity to SF-like features, being essential for effective feature extraction, becomes their weakness. All in all, Table~\ref{mtasr}, and Fig.~\ref{imagenet_trans} and \ref{cifar_trans}, demonstrate strong cross-model transferability, under both pretrained and robust regimes.

\textbf{Baseline comparisons (and a note on visual fidelity).} We compare the effectiveness and transferability of our attack against AutoAttack and F-PGD, on pretrained and RobustBench models, and report the results in the heatmaps of Fig.~\ref{imagenet_trans} (ImageNet) and  Fig.~\ref{cifar_trans} (CIFAR100), with $\varepsilon=4/255$ and $\varepsilon=8/255$, respectively, for the baselines, and $\edgf=\eta\varepsilon$ for our attack. For clarity, we also present the averaged off-diagonal ASR of each heatmap in Table~\ref{mtasr}. We observe that the proposed attack achieves an almost identical self-model ASR to AutoAttack, which solves \eqref{maxattack}, thus confirming our motivation for designing the optimization problem \eqref{space} to derive SF attacks. In contrast, F-PGD follows the formulation of \eqref{freqattack} - \eqref{maskgrad}, without resulting from an optimization problem. Our attack consistently outperforms the baselines by means of source-target ASR, across pretrained and robust models alike, implying higher transferability, even against strongly defended RobustBench models. Overall results demonstrate that structured, overcomplete SF transforms, by aligning more naturally with data geometry, enable attacks to capture shared spectral sensitivities of architectures, more effectively than purely spatial-/frequency-based attacks. These findings also pinpoint to a concrete path for creating more robust models: incorporate transform-domain priors, enjoying additional structure, in adversarial training. For instance, our attack is a solution to \eqref{space}, so the latter can naturally replace the standard $\ell_p$-constrained inner maximization in adversarial training, to yield $\min_\theta\max_{\delta\in\mathcal{C}_{M_D}}\,\mathcal{L}(x+\delta,y;\theta)$, for  parameters $\theta$; we leave this strategy for future work (cf. Section~\ref{conclusion}).

We also examine visual quality, via SSIM, of adversarial samples of RobustBench models, illustrated in Fig.~\ref{shampoos} (ImageNet) and \ref{flowers} (CIFAR100), with $\varepsilon=8/255$ and $\varepsilon=16/255$, respectively, for the baselines, and $\edgf=\eta\varepsilon$ for our attack; we visualize all three attacks in the spatial and frequency domain, in Fig.~\ref{spectrashampoos} and \ref{spectraflowers}, respectively. While all methods retain a reasonable visual quality, our attack outperforms the baselines, in terms of ASR and SSIM. We observe that AutoAttack leads to pixel-aligned variations, reflecting the $\ell_\infty$-norm geometry, while F-PGD produces smoother spatial distortions, with energy concentrated in specific frequencies. On the contrary, our method distributes energy along structured SF orientations, confirming our motivation for designing a perturbation set to reflect the DGF geometry and overcompleteness. More broadly, Fig.~\ref{imagenet} and \ref{cifar100} showcase that attacks resulting from a maximization problem, with a constraint set relying on Gabor frames being closely related to human perception (cf. Section~\ref{back}), can be highly effective, while preserving visual fidelity.

\section{Conclusion \& future work}
\label{conclusion}
In this paper, we introduced a principled methodology for explaining the success of frequency-based adversarial attacks. Our analysis relied on introducing an optimization problem over a newly-defined perturbation set, incorporating structured, non-orthogonal transforms, allowing for flexible data representation. We proved that the solution to the optimization problem is acquired through a projection on the aforesaid set, resulting in an adversarial attack of increased effectiveness and transferability, as supported by relevant experiments. Particularly, our proposed method exploited model vulnerabilities beyond traditional high-/low-frequency assumptions, highlighting the pivotal role of transform geometry in adversarial robustness. Due to its grounded justification and preliminary empirical success, our work opens promising future directions. For instance, it would be intriguing to develop a larger empirical evaluation, with more datasets, attacking settings, and different reweighting strategies for the DGF, as well as couple our method with adversarial training, to establish a complete benchmark for transform-aware defenses that align with perceptual and representational structure. Furthermore, it would be fruitful to extend our framework to other families of structured transforms, e.g., those of learnable adaptive transforms \cite{mairal2008supervised}, and examine how these transforms could relate to shared spectral biases of diverse architectures.

\bibliography{ref}


\begin{appendices}

\section{Details of Section~\ref{back}}
\label{mainappen}
We present below the mathematical definition of a frame.
\begin{definition}[\cite{casazza2012finite}]
    A set $(\psi_i)_{i=1}^N$ of $N$ vectors in $\fn$, $N\geq n$, is a frame for $\fn$ if and only if it is a spanning set for $\fn$. Equivalently, if the inequalities
    \begin{equation}\label{frameineq}
    A\|x\|_2^2\leq\sum_{i=1}^N\lvert\langle x,\psi_i\rangle\rvert^2\leq B\|x\|_2^2
\end{equation}
hold for all $x\in\fn$, for some $0<A\leq B<\infty$ (frame bounds), then $(\psi_i)_{i\in N}$ is a frame for $\fn$.
\end{definition}

While frames are usually studied in the 1D case, e.g., for $x\in\rn$, they can be naturally extended to 2D data, by relying on tensor algebra of frames, since the tensor product of two frames is a frame itself \cite{paukner2007foundations}. As a result, for $x\in\Rnn$ and a DGF with associated analysis operator $\Psi\in\cNn$ as in Definition~\ref{dgf}, we have
\begin{equation}
    (\Psi\otimes\Psi)\mathrm{vec}(x)=\Psi x\Psi^T=w\in\mathbb{C}^{N\times N},
\end{equation}
with $\Psi\otimes\Psi$ being the analysis operator associated to a DGF in $\cNn\otimes\cNn\cong\mathbb{C}^{N^2\times n^2}$. In this case, we call $\Psi$ a 2D separable DGF analysis operator. Similarly to Remark~\ref{expansion}, 2D frames enable lossless expansions \cite{paukner2007foundations}:
\begin{equation}
    x=\Psi^+w(\Psi^+)^T,
\end{equation}
for $\Psi^+\in\mathbb{C}^{n\times N}$ being the pseudo-inverse of $\Psi\in\cNn$.

Based on the afore-described framework, we can readily extend the methodology we propose in Section~\ref{main}, to account for images $x\in\Rnn$, with adversarial attacks $\delta\in\Rnn$, so $w=\Psi\delta\Psi^T\in\mathbb{C}^{N\times n}$ is an admissible frame attack, due to Lemma~\ref{admlemma}; we give more details in Appendix~\ref{2dcase}.

\section{Proofs of Section~\ref{main}}
\label{proofappen}
\subsection{Proof of Lemma~\ref{admlemma}}
\label{admappen}
\begin{proof}
For the first bullet-point, let us assume that $\delta$ is a spatial attack of level $\varepsilon$ in $\fn$, so that $\|\delta\|_{\fn}\leq\varepsilon$. Then, for $w=\Psi\delta$, we have
\begin{equation*}
    \|w\|_{\fN}=\|\Psi\delta\|_{\fN}\leq\|\Psi\|_{\mathrm{op}}\varepsilon.
\end{equation*}
By definition of $\Psi$, we conclude that $w$ is an admissible frame attack of level $\|\Psi\|_{\mathrm{op}}\varepsilon$.\\
For the second bullet-point, let us suppose that $w=\Psi\delta$ is an admissible frame attack of level $\varepsilon$, so it holds $\|w\|_{\fN}\leq\varepsilon$. Then, due to the lossless expansion offered by the frame, we get $\|\delta\|_{\fn}=\|\Psi^+w\|_{\fn}\leq\|\Psi^+\|_{\mathrm{op}}\varepsilon$. The proof is complete.
\end{proof}

\begin{figure}[t!]
    \centering
    \includegraphics[width=.9\textwidth]{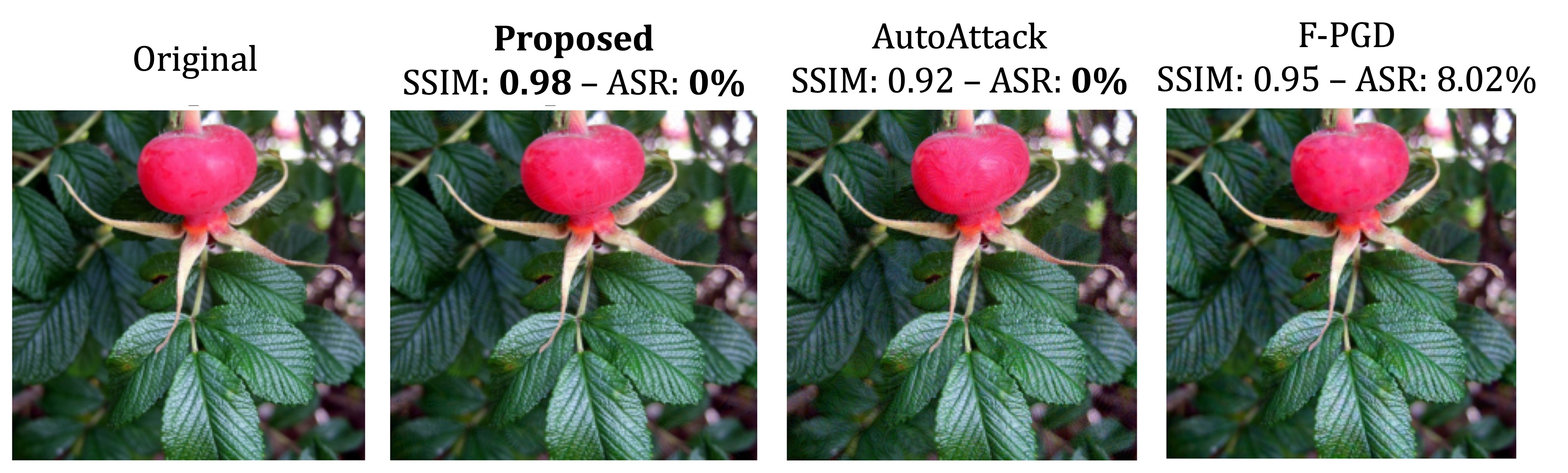}
    \includegraphics[width=.9\textwidth]{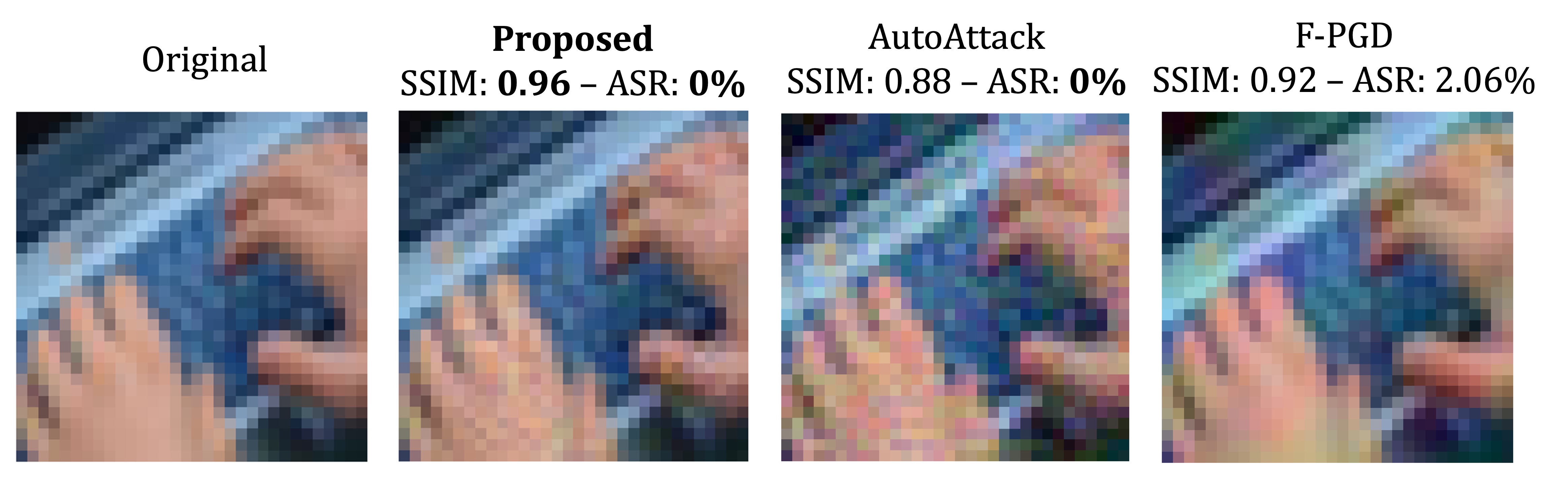}
    \caption{Examples of images and their adversarial counterparts, with \textbf{SSIM} quality and \textbf{ASR} attack efficiency values, evaluated on a pretrained version of RegNet for \textbf{ImageNet (top)} and on a pretrained version of ResNet for \textbf{CIFAR100 (bottom)}, for $\varepsilon=8/255$ and $\varepsilon=16/255$, respectively, on the baselines, and $\edgf=\eta8/255$ $\edgf=\eta16/255$, respectively, on our attack. Bold letters indicate the method with the best performance in terms of both evaluation metrics.}
    \label{basefigkeyboards}
\end{figure}

\subsection{Proof of derivation of \eqref{framepgd}}
\label{pga}
\begin{proof}
    We aim to solve \eqref{space} by a variant of projected gradient ascent. To that end, we seek the steepest ascent direction under the geometry induced by the perturbation set \eqref{c1}. At iteration $k$, we set $\Tilde{g}^k=\nabla_\delta\loss(\tilx+\delta^k,y)$ and solve
\begin{equation}
\label{ascent}
    u^\star=\arg\max_{u\in\rn}(\Tilde{g}^k)^Tu\quad\text{such that}\quad u^TM_Du\leq1.
\end{equation}
By definition of \eqref{percmetric} and due to the injectivity of $\Psi$, $M_D$ is Hermitian and positive definite, and thus invertible. Moreover, it holds $u^TM_Du=u^T\Re(M_D)u$, for any $u\in\rn$. 

We set the Langrangian of \eqref{ascent}:
\begin{equation}
    L(u;\lambda)=(\Tilde{g}^k)^Tu - \lambda(u^T\Re(M_D)u-1),\qquad\lambda\geq0.
\end{equation}
By the KKT conditions, we have
\begin{equation}
    \nabla_uL(u;\lambda)=0\implies\Tilde{g}^k=2\lambda\Re(M_D)\implies u=\frac{(\Re(M_D))^{-1}\Tilde{g}^k}{2\lambda}.
\end{equation}
Imposing the active constraint $u^T\Re(M_D)u=1$ yields
\begin{equation}
    \frac{(\Tilde{g}^k)^T(\Re(M_D))^{-1}\Tilde{g}^k}{(2\lambda)^2}\implies2\lambda=\sqrt{(\Tilde{g}^k)^T(\Re(M_D))^{-1}\Tilde{g}^k}.
\end{equation}
Hence, the steepest ascent direction is given by
\begin{equation}
\label{direction}
    u^\star=\frac{(\Re(M_D))^{-1}\Tilde{g}^k}{\sqrt{(\Tilde{g}^k)^T(\Re(M_D))^{-1}\Tilde{g}^k}}.
\end{equation}
Then, deploying projected gradient ascent in the direction of \eqref{direction} yields \eqref{framepgd}, which complements the proof.
\end{proof}
\subsection{Proof of Proposition~\ref{projection}}
    \label{projappen}
    \begin{proof}
        Using standard optimization theory \cite{boyd2004convex}, we set a Langrangian
        \begin{align}
            L(\delta';\lambda)=&\frac{1}{2}\|\delta'-\delta\|_2^2+\lambda((\delta')^T M_D\delta' - \varepsilon^2)\notag\\
            =&\frac{1}{2}\|\delta'-\delta\|_2^2+\lambda((\delta')^T\Re( M_D)\delta' - \varepsilon^2),
        \end{align}
         for $\lambda\geq0$. By KKT conditions, we get
        \begin{align}
            \nabla_{\delta'}L=0 &\implies(\delta'-\delta)+2\lambda \Re( M_D)\delta'=0\notag\\
            &\implies\delta'=(I+2\lambda \Re( M_D))^{-1}\delta,
        \end{align}
        for $(\delta')^T \Re( M_D)\delta'\leq \varepsilon^2$, $\lambda((\delta')^T \Re( M_D)\delta' - \varepsilon^2)=0$. Since $\Re(M_D)$ is invertible, so is $I+2\lambda\Re(M_D)$, for all $\lambda\geq0$. \\
        If $\delta\in\pertset_{M_D}$, then it satisfies the constraint and $\lambda=0$, yielding $\delta'=\delta$. If $\delta\notin\pertset_{M_D}$, then the constraint is active, $\lambda>0$ and $\delta'=(I+2\lambda\Re( M_D))^{-1}\delta$, with $\lambda$ chosen such that $(\delta')^T\Re(M_D)\delta'=\varepsilon^2$. Overall, we have
        \begin{align}
        \Pi_{\pertset_{M_D}}(\delta)
            &=\begin{cases}
        \delta & \text{if } \delta^T M_D \delta \leq \varepsilon^2 \\
        (I + 2\lambda\Re(M_D))^{-1}\delta & \text{if } \delta^T M_D \delta > \varepsilon^2,
        \end{cases}
        \end{align}
        which completes the first part of the proof.\\
        Due to the positive definitiveness of $M_D$ (and $\Re( M_D)$), the set $\pertset_{M_D}$ is non-degenerate and the projection is unique for all $\delta$. To compute $\lambda$, we diagonalize $\Re(M_D)$, so that $\Re(M_D)=U\Lambda U^T$, where $\Lambda$ is a diagonal matrix with entries $\mu_i>0$, $i=1,\dots,n$, and set $\zeta=U^T\delta$. We define
        \begin{equation}
            \phi(\lambda):=\sum_{i=1}^n\frac{\mu_i\zeta_i^2}{(1+2\lambda\mu_i)^2}-\varepsilon^2;
        \end{equation}
        it can be proven that $\phi$ is a strictly decreasing, continuous function. Moreover, $\phi(0)>0$ and
        \begin{equation}
            \phi(\lambda)\overset{\lambda\rightarrow\infty}{\longrightarrow}-\varepsilon^2<0.
        \end{equation}
        By the intermediate value theorem and strict monotonicity, there is a unique $\lambda^*>0$ satisfying $\phi(\lambda^*)=0$; we find $\lambda^*$ by solving $\phi(\lambda)=0$ using the bi-section method.
    \end{proof}

\begin{table}[t!]
\centering
\caption{Average \textbf{ASR} (in \%) across source-target model pairs, computed from the off-diagonal entries of the heatmaps of
Fig.~\ref{transferability_imagenet_heatmaps_double} and Fig.~\ref{transferability_cifar100_heatmaps_double}. \textbf{Bold letters indicate better cross-model transferability}.}
\label{mtasr_high_noise}
\scalebox{0.7}{\begin{tabular}{lcccc}
\toprule
& \multicolumn{2}{c}{ImageNet} & \multicolumn{2}{c}{CIFAR100} \\
\cmidrule(lr){2-3} \cmidrule(lr){4-5}
Attack & Pretrained & RobustBench & Pretrained & RobustBench \\
\midrule
\textbf{Proposed} (spatial-frequency-based) & \textbf{75.4} & \textbf{63.5} & \textbf{96.9} & \textbf{83.8} \\
AutoAttack (spatial-based) & 44.0          & 41.6          & 89.6          & 64.9          \\
F-PGD (frequency-based) & 42.7          & 39.7          & 77.6          & 58.5          \\
\bottomrule
\end{tabular}}
\end{table}

\subsection{Extension to a 2D setting}
\label{2dcase}
In Section~\ref{main}, our proposed methodology considers 1D data; this is essential, in order to scale to 2D data like the ImageNet and CIFAR100 coloured images, which we process batch- and channel-wise. As such, we follow the 2D extension we describe in Appendix~\ref{mainappen}. We capitalize all vectors to shift to matrix notation, use a 2D separable DGF analysis operator, and perform separable, i.e., row- and column-wise, weighting, using $D$ \eqref{percmetric}. By turning to a Frobenius $\|\cdot\|_F$-norm instead of the $\ell_2$-norm for \eqref{perc1}, we extend the latter in the 2D case as $\|\Psi \Delta\Psi^T\|_D^2=\|W\|_D^2=\mathrm{tr}(W^TDWD)$, or equivalently $\|W\|_D=\sqrt{\mathrm{vec}(W)^T(D\otimes D)\mathrm{vec}(W)}=\sqrt{\mathrm{vec}(\Delta)^T(M_D\otimes M_D)\mathrm{vec}(\Delta)}$, where $\mathrm{tr}(\cdot)$ denotes the trace of a matrix, and $M_D$ is as in \eqref{mmatrix}. Then, by slightly abusing notation and denoting the perturbation set for both 1D and 2D cases as $\pertset_{M_D}$, we obtain for \eqref{c1} and \eqref{framepgd}
\begin{align}
    \pertset_{M_D}&=\{\Delta\in\Rnn:\,\mathrm{tr}(\Delta^TM_D\Delta M_D)\leq\varepsilon^2\},\\
    \label{2dstep}\Delta^{k+1}&=\Pi_{\pertset_{M_D}}\bigg(\Delta^k+\gamma\frac{(\Re(M_D))^{-1}\widetilde{G}^k(\Re(M_D))^{-1}}{\sqrt{(\widetilde{G}^k)^T(\Re(M_D))^{-1}\widetilde{G}^k(\Re(M_D))^{-1}}}\bigg),
\end{align}
respectively, where $\widetilde{G}^k=\nabla_\Delta\loss(X+\Delta^k,Y)$, and in \eqref{2dstep} we have used the identity $(\Re(M_D)\otimes\Re(M_D))^{-1}=(\Re(M_D))^{-1}\otimes (\Re(M_D))^{-1}$. Then, the corresponding projection of Proposition~\ref{projection} is given by
\begin{align}
    \Pi_{\pertset_{M_D}}(\Delta)=\arg\min_{\Delta'\in\pertset_{M_D}}\|\Delta'-\Delta\|_F
    =\begin{cases}
\Delta, &\Delta\in\pertset_{M_D}\\
(I + 2\lambda\Re(M_D))^{-1}\Delta(I + 2\lambda\Re(M_D))^{-1}, & \Delta\notin\pertset_{M_D}
\end{cases},
\end{align}
which easily follows by mutatis mutandis  adapting the proof presented in Appendix~\ref{projappen}, based on the 2D separability of all associated matrices.

\begin{figure}[t!]
    \centering
    \includegraphics[width=\textwidth]{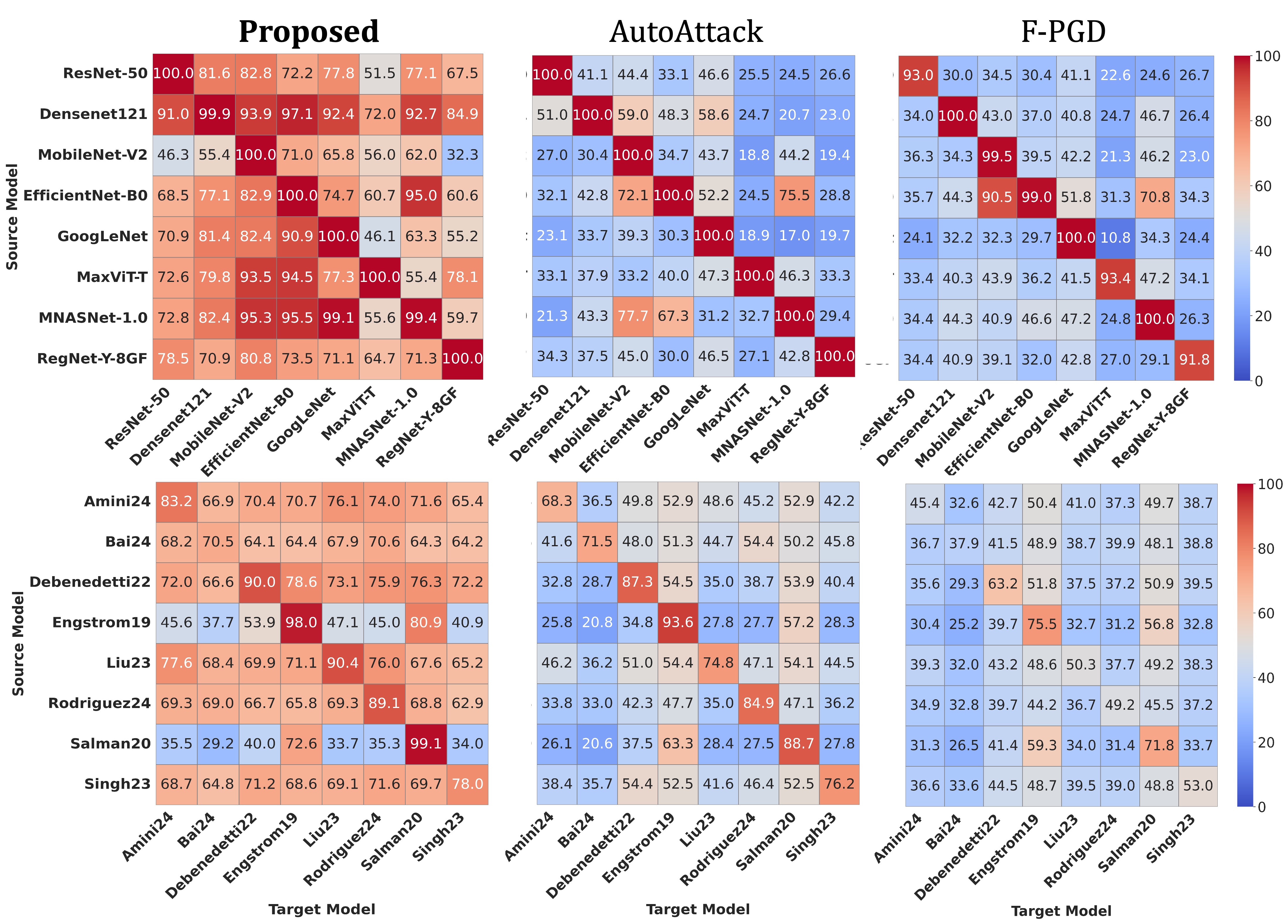}
    \caption{\footnotesize \textbf{ASR} (in $\%$) with $\varepsilon=8/255$ and $\edgf=\eta8/255$, on \textbf{ImageNet}, for different \textbf{pretrained (top)} and \textbf{RobustBench (bottom)} models. \textbf{Left to right: proposed attack, AutoAttack, F-PGD}. \textbf{Color intensity indicates higher transferability}. We observe that our proposed attack is consistently more transferable than the baselines, as we double the attack level (cf. Fig.~\ref{imagenet_trans}), across a variety of models. These results conform with our motivation deriving attacks from a principled methodology, based on structured transforms.}   \label{transferability_imagenet_heatmaps_double}
\end{figure}

\section{Additional experiments details \& results}
\label{expappen}
To construct the DGF, we need to use an adequate so-called window vector $g$ (cf. Definition~\ref{dgf}), on which we perform the SF shifts. We select the Hann window vector $g\in\rn$, defined as 
\begin{equation}
    g_j=0.5\left(1-\cos\left(\frac{2j\pi}{n-1}\right)\right),
\end{equation}
$j=0,\dots,n-1$, commonly used for image-processing applications. The DGF parameters determine the SF resolution of the representation, and the frame overcompleteness, ensuring $\alpha\beta<n$. Choosing $\alpha=1$ enables maximal spatial sampling, allowing the representation to capture fine localized structures, and $\beta=16$ (for CIFAR100) and $\beta=112$ (for ImageNet) provides a balanced frequency resolution, that preserves relevant spectral information. This configuration yields a stable, overcomplete representation capable of capturing structured SF patterns, being central to our perturbation set design. We also set $\tau=10^{-6}$ for the constant ensuring the positive-definitiveness of the reweighting matrix $D$ defined in \eqref{percmetric}. We also calculate the scaling factor $\eta$ separately for each dataset, as it is dimension-dependent, due to the appearance of $M_D$.

\begin{figure}[t!]
    \centering
    \includegraphics[width=\textwidth]{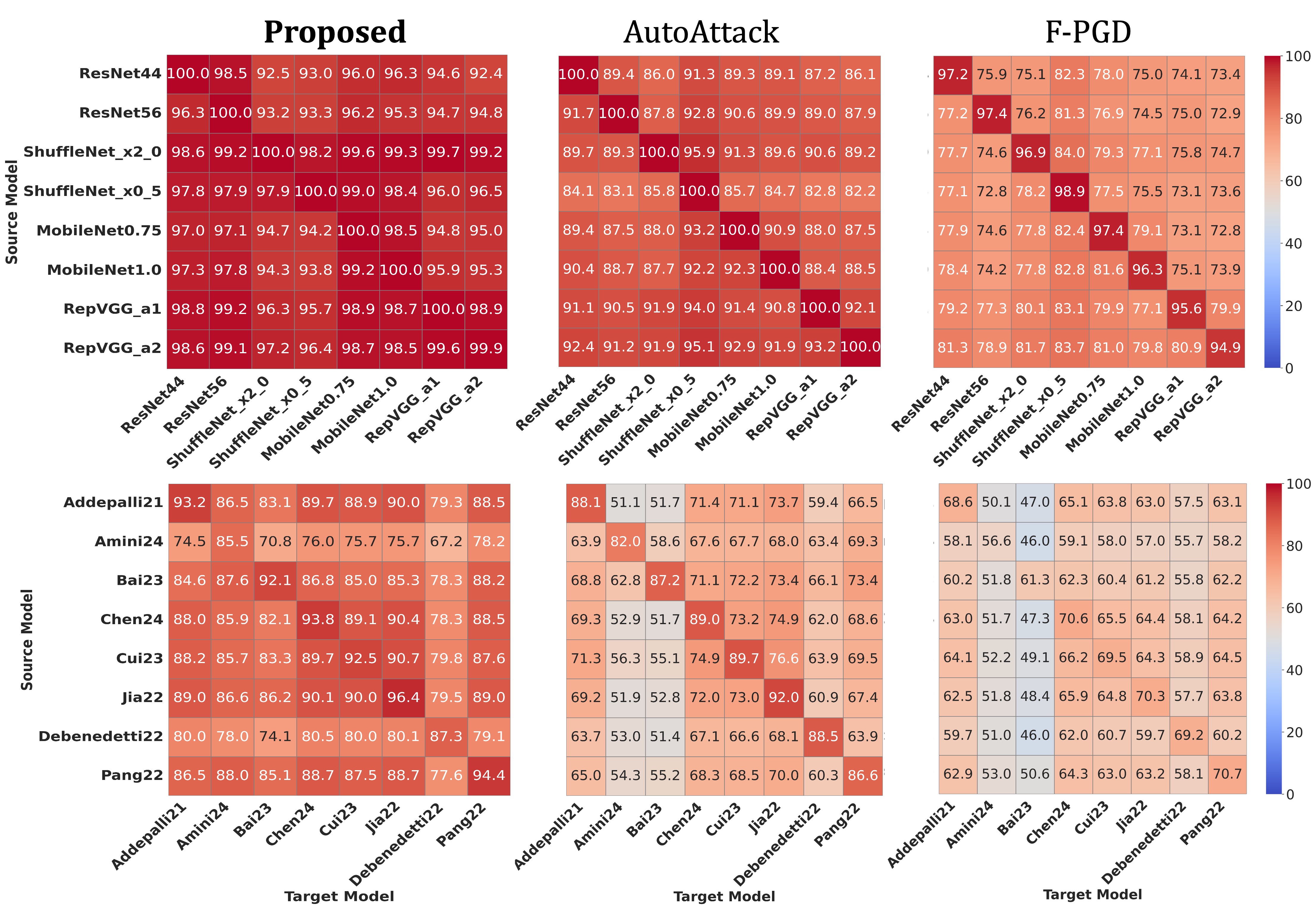}
    \caption{\footnotesize \textbf{ASR} ($\%$) on \textbf{CIFAR100}, with $\varepsilon=16/255$ for the baselines and $\edgf=\eta\varepsilon$ for the proposed attack, on different \textbf{pretrained (top)} and \textbf{RobustBench (bottom)} models. \textbf{Color intensity indicates higher transferability}. We notice that \textbf{our attack (left)} is highly transferable, even more than on ImageNet (cf. Fig.~\ref{transferability_imagenet_heatmaps_double}), outperforming the \textbf{baselines (middle and right)}, for both pretrained and robust models.}
    \label{transferability_cifar100_heatmaps_double}
\end{figure}

All experiments are implemented in PyTorch \cite{pytorch}, with a batch size 8 for ImageNet and 32 for CIFAR100, and conducted on two high-performance computing clusters, using NVIDIA A100 (40GB VRAM) and H100 (80GB VRAM) GPUs.  Under these configurations, the average execution time was consistently under 20 minutes. An exception was the transferability analysis of AutoAttack, which required exceedingly more computational time, which is natural, by definition of the AutoAttack's scheme. While AutoAttack has a predefined number of iterations for each of the attacks involved in its design (e.g., it may be $K=100$), we set the number of iterations to $K=20$ for both our proposed attack and the F-PGD, to enable a fair comparison. Relative to F-PGD, the additional cost of the proposed attack arises from the implementation of $\Re(M_D)$, which requires applying its inverse on the gradient of the loss and projecting onto the perturbation set at each iteration. All computations exploit the separable structure of the transforms and are implemented via forward and adjoint operations. The projection step reduces to applying $(I+2\lambda\Re(M_D))^{-1}$ from the left and right via structured linear solves. The scalar parameter $\lambda$ is computed using a bisection procedure, which has logarithmic cost and introduces only a negligible overhead compared to the cost of the transform operations. In contrast, AutoAttack combines multiple iterative attacks and requires substantially more gradient evaluations, leading to significantly higher runtime. In practice, the proposed attack incurs a moderate overhead compared to F-PGD due to the additional structured solves, but remains considerably more efficient than AutoAttack. Overall, it provides a favorable trade-off between computational cost and attack performance. Upon acceptance, our code will become publicly available in a corresponding Github repository, to support open science practices. For experimental completeness, we present below complementary experiments to those of Section~\ref{results}.

Similarly to Fig.~\ref{imagenet_trans} and Fig.~\ref{cifar_trans}, we present transferability heatmaps, serving as comparisons between the proposed attack and the baselines, on pretrained and RobustBench models. We illustrate the results in Fig.~\ref{transferability_imagenet_heatmaps_double} for ImageNet, with $\varepsilon=8/255$ for the baselines and $\edgf=\eta\varepsilon$ for our attack, and in Fig.~\ref{transferability_cifar100_heatmaps_double} for CIFAR100, with $\varepsilon=16/255$ for the baselines and $\edgf=\eta\varepsilon$ for our attack. To better reflect transferability, we also present in Table~\ref{mtasr_high_noise}, the average cross-model ASR, calculated according to each heatmap of Fig.~\ref{transferability_imagenet_heatmaps_double} and Fig.~\ref{transferability_cifar100_heatmaps_double}. We observe that, for the pretrained models, on both datasets, our attack is almost identically effective to AutoAttack, in terms of self-model ASR, demonstrating that attacks derived from an appropriate optimization problem, over a specified perturbation set, can outperform frequency-based attacks that only follow the formulation of \eqref{freqattack} and \eqref{maskgrad}, without being accompanied by a grounded maximization problem. Interestingly, while effectiveness and transferability of both baselines increase as we transfer to a more noisy regime than that of Fig.~\ref{imagenet_trans} and Fig.~\ref{cifar_trans}, our attack still outperforms both baselines, consistently for both datasets, in terms of the transferability rate. This further justifies our motivation for designing an optimization problem whose perturbation constraint set is endowed with a DGF, which better aligns with inherent SF patterns of natural images, thus the proposed attack captures more easily than purely spatial-/frequency-based attacks the shared spectral sensitivities of a range of models, both pretrained and defended.

Finally, we illustrate original and adversarial image samples from ImageNet and CIFAR100, in Fig.~\ref{basefigkeyboards} top and bottom, respectively, pertaining to the proposed attack, AutoAttack, and F-PGD, on pretrained models, as a means of examining whether our proposed methodology has the same effect depicted in Fig.~\ref{imagenet} and Fig.~\ref{cifar100}, which relied on RobustBench models evaluations. For these image samples, we separately present in Fig. the attacks illustrated in the spatial and frequency domain, so as to elaborate on the representation ability of each of the three attack methods. Firstly, we observe that our attack outperforms the baselines, by scoring a mildly higher SSIM value, for both datasets, albeit scoring the highest possible ASR, along with AutoAttack. Additionally, while all images retain a reasonable visual quality, they exhibit structural differences in both spatial- and frequency-based representations, highlighting that the attack effectiveness is not solely determined by operating in the frequency domain, but rather by how perturbations are organized and constrained during optimization. This phenomenon is further supported by Fig.~\ref{spectraappen}, where we notice that the SF localization of the proposed attack is more evident, spread strategically in the transformed-domain, than that of the baselines. This observation conforms with our intuition of deploying Gabor frames, which induce norms that align better with natural image patterns than standard-$\ell_p$ balls, thereby leading to an attack that more flexibly handles image geometry to be highly effective, without sacrificing image quality.

\begin{figure}[t!]
    \centering
    \begin{subfigure}[h]{0.7\textwidth}
        \centering
        \includegraphics[width=\textwidth]{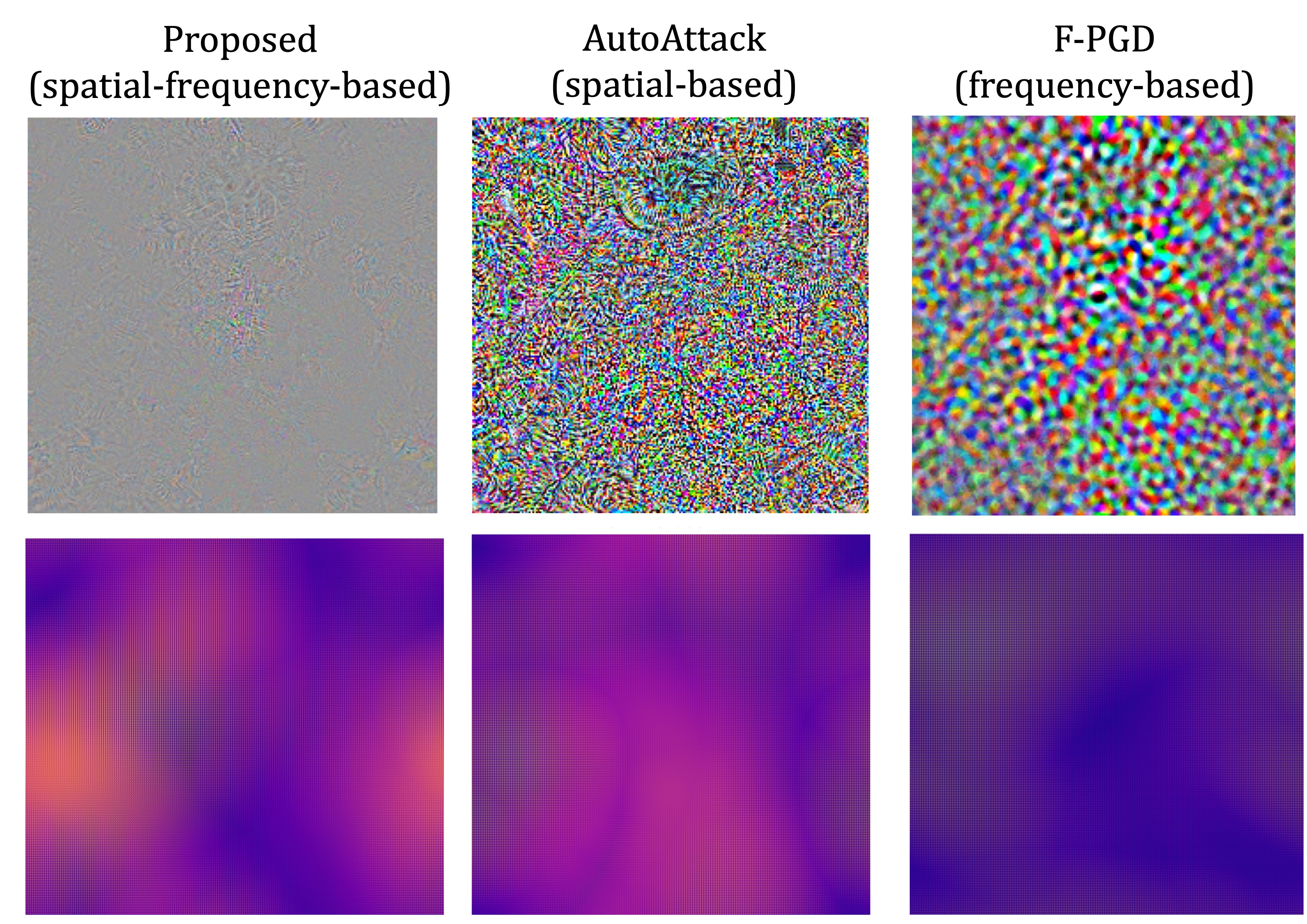}
        \caption{ImageNet}
    \end{subfigure}%
    
    \begin{subfigure}[h]{0.7\textwidth}
    \centering
    \includegraphics[width=\textwidth]{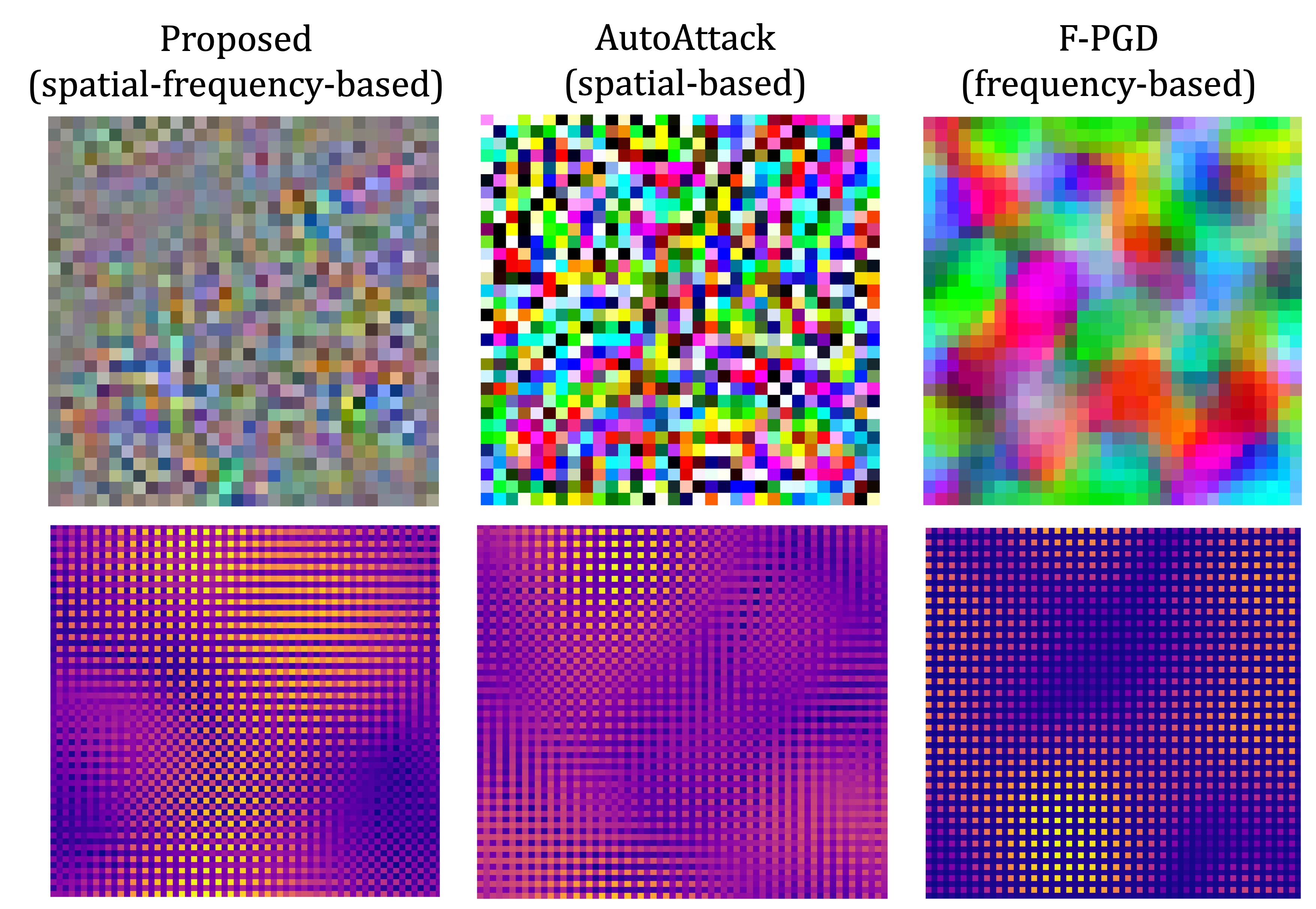}
    \caption{CIFAR100}
    \end{subfigure}
    \caption{Exemplary \textbf{spatial (top)} and \textbf{frequency (bottom)} representation of the \textbf{proposed attack and the two baselines}, corresponding to the adversarial images of Figure~\ref{basefigkeyboards}. Similarly to Fig.~\ref{spectrashampoos} and Fig.~\ref{spectraflowers}, we observe that the SF ability of our attack, for both datasets, provides a balanced representation in the transformed space, while the baselines rely mostly on spatial and frequency coordinates, thus missing the flexibility offered by the DGF.}
    \label{spectraappen}
\end{figure}

\section{Impact Statement}
\label{impactstate}
This work advances the understanding of adversarial vulnerabilities in deep learning models, by introducing a principled framework for generating frequency-based attacks, with our primary goal being the improvement of robustness evaluation and creation of better defenses.

Specifically, our framework can help researchers and practitioners better understand fundamental weaknesses in neural networks, particularly those that persist across different architectures. By demonstrating that attacks constrained in structured transform domains can be highly effective -- even against adversarially trained models -- we highlight the need for more comprehensive defense mechanisms, accounting for highly structured frequency-based attacks. This could lead to more robust models in safety-critical applications.

Like other adversarial attack research, this work could potentially be misused to compromise vision systems. However, we believe the benefits of openly studying these vulnerabilities outweigh the risks, as understanding attack mechanisms is essential for developing effective defenses. Our work does not introduce fundamentally new attack capabilities, but rather provides a more principled understanding of existing frequency-based attacks. By grounding adversarial behavior in a mathematically principled framework, our study provides tools for interpreting model vulnerabilities in spectral terms. We anticipate that these insights can encourage the development of robustness benchmarks and training schemes that better account for structure in real-world data representations.

\end{appendices}


\newpage
\section*{NeurIPS Paper Checklist}
\begin{enumerate}

\item {\bf Claims}
    \item[] Question: Do the main claims made in the abstract and introduction accurately reflect the paper's contributions and scope?
    \item[] Answer: \answerYes{}.
    \item[] Justification: All the claims found in the abstract and introduction are properly supported in the main results of Sec.~\ref{main} and the experimental results of Sec.~\ref{exp}. Additionally, proofs for our derived theory can be found in Appendix~\ref{proofappen} and  experimental extensions in Appendix~\ref{expappen}.

\item {\bf Limitations}
    \item[] Question: Does the paper discuss the limitations of the work performed by the authors?
    \item[] Answer: \answerYes{}.
    \item[] Justification: Throughout the paper, we state the limitations of our work, among which the fact that a more thorough empirical evaluation is needed in the future for our proposed methodology (cf. Sec.~\ref{exp}).

\item {\bf Theory assumptions and proofs}
    \item[] Question: For each theoretical result, does the paper provide the full set of assumptions and a complete (and correct) proof?
    \item[] Answer: \answerYes{}
    \item[] Justification: For the course of our theoretical results presented in Sec.~\ref{main}, we state minimal and justified assumptions, to clarify every possible dependency of the problem. All associated proofs can be found in a complete form at Appendix~\ref{proofappen}.

    \item {\bf Experimental result reproducibility}
    \item[] Question: Does the paper fully disclose all the information needed to reproduce the main experimental results of the paper to the extent that it affects the main claims and/or conclusions of the paper (regardless of whether the code and data are provided or not)?
    \item[] Answer: \answerYes{}
    \item[] Justification: In Sec.~\ref{settings}, we provide details on the experimental setup of the paper leading to the corresponding experimental results. We also employ example datasets used in papers that are close to our work, while in Sec.~\ref{expappen}, we outline more experimental details, including choice of hyperparameters for reproducibility purposes.

\item {\bf Open access to data and code}
    \item[] Question: Does the paper provide open access to the data and code, with sufficient instructions to faithfully reproduce the main experimental results, as described in supplemental material?
    \item[] Answer: \answerYes{}
    \item[] Justification: Upon acceptance, we will provide a link to a public github repository with pytorch code, and sufficient documentation for reproducibility of all the experimental results that accompany the paper.

\item {\bf Experimental setting/details}
    \item[] Question: Does the paper specify all the training and test details (e.g., data splits, hyperparameters, how they were chosen, type of optimizer) necessary to understand the results?
    \item[] Answer: \answerYes{}
    \item[] Justification: We briefly describe the main experimental setup in Section~\ref{settings}, and then further elaborate on all details in Appendix~\ref{expappen}.

\item {\bf Experiment statistical significance}
    \item[] Question: Does the paper report error bars suitably and correctly defined or other appropriate information about the statistical significance of the experiments?
    \item[] Answer: \answerYes{}.
    \item[] Justification:  All experiments are repeated at least 10 times with different random seeds, and results are reported as averages over these runs. Additionally, we compute the standard deviation across runs, which we found to be consistently small (typically below 0.05 in absolute accuracy), and thus corresponding error bars are omitted from the figures for the sake of readability.

\item {\bf Experiments compute resources}
    \item[] Question: For each experiment, does the paper provide sufficient information on the computer resources (type of compute workers, memory, time of execution) needed to reproduce the experiments?
    \item[] Answer: \answerYes{}
    \item[] Justification: For the course of our experiments, we present detailed descriptions on the computer resources in Appendix~\ref{expappen}.
    
\item {\bf Code of ethics}
    \item[] Question: Does the research conducted in the paper conform, in every respect, with the NeurIPS Code of Ethics \url{https://neurips.cc/public/EthicsGuidelines}?
    \item[] Answer: \answerYes{}
    \item[] Justification: We preserved anonymity in our submission. Our submission abides by the NeurIPS Code of Ethics.

\item {\bf Broader impacts}
    \item[] Question: Does the paper discuss both potential positive societal impacts and negative societal impacts of the work performed?
    \item[] Answer: \answerYes{}
    \item[] Justification: We discuss potential societal impacts of our work in Appendix~\ref{impactstate}.
    
\item {\bf Safeguards}
    \item[] Question: Does the paper describe safeguards that have been put in place for responsible release of data or models that have a high risk for misuse (e.g., pre-trained language models, image generators, or scraped datasets)?
    \item[] Answer: \answerNA{}
    \item[] Justification: : Our paper does not pose any such risks.

\item {\bf Licenses for existing assets}
    \item[] Question: Are the creators or original owners of assets (e.g., code, data, models), used in the paper, properly credited and are the license and terms of use explicitly mentioned and properly respected?
    \item[] Answer: \answerYes{}
    \item[] Justification: All original owners of assets are properly credited, e.g., we cite the original papers of the baseline attacks we investigate for comparison with our proposed attack.

\item {\bf New assets}
    \item[] Question: Are new assets introduced in the paper well documented and is the documentation provided alongside the assets?
    \item[] Answer: \answerNA{}
    \item[] Justification: We do not release new assets.

\item {\bf Crowdsourcing and research with human subjects}
    \item[] Question: For crowdsourcing experiments and research with human subjects, does the paper include the full text of instructions given to participants and screenshots, if applicable, as well as details about compensation (if any)? 
    \item[] Answer: \answerNA{}
    \item[] Justification: No crowdsourcing or human subjects were involved in the experiments conducted for this paper.

\item {\bf Institutional review board (IRB) approvals or equivalent for research with human subjects}
    \item[] Question: Does the paper describe potential risks incurred by study participants, whether such risks were disclosed to the subjects, and whether Institutional Review Board (IRB) approvals (or an equivalent approval/review based on the requirements of your country or institution) were obtained?
    \item[] Answer: \answerNA{}
    \item[] Justification: No IRB, crowd-sourcing, or research with human subjects were involved in the experiments conducted for this paper.

\item {\bf Declaration of LLM usage}
    \item[] Question: Does the paper describe the usage of LLMs if it is an important, original, or non-standard component of the core methods in this research? Note that if the LLM is used only for writing, editing, or formatting purposes and does \emph{not} impact the core methodology, scientific rigor, or originality of the research, declaration is not required.
    \item[] Answer: \answerNA{}
    \item[] Justification: The core methods developed during the present research do not involve LLMs as any important, original, or non-standard components.

\end{enumerate}

\end{document}